\documentclass{article}

\usepackage[preprint]{neurips_2022}

\usepackage[utf8]{inputenc} 
\usepackage[T1]{fontenc}    
\usepackage{hyperref}       
\usepackage{url}            
\usepackage{booktabs}       
\usepackage{amsfonts}       
\usepackage{nicefrac}       
\usepackage{microtype}      
\usepackage{xcolor}         
\usepackage{todonotes}
\usepackage{amsthm}

\title{Near-Optimal Collaborative Learning in Bandits}

\author{%
  Cl\'{e}mence R\'{e}da \\
  Universit\'{e} Paris Cité,\\
  Inserm, NeuroDiderot,\\
  F-$75019$ Paris, France \\
  \texttt{clemence.reda@inria.fr} \\
  \And
  Sattar Vakili \\
  MediaTek Research, \\
  Cambourne Business Park, \\
  CB23 6DW, United Kingdom \\
   \texttt{sattar.vakili@mtkresearch.com} \\
   \And
 {E}milie Kaufmann\\
 Universit\'{e} de Lille, \\
 CNRS, Inria, Centrale Lille, \\
 UMR $9189$ CRIStAL, \\
 F-$59000$ Lille, France \\
  \texttt{emilie.kaufmann@univ-lille.fr} 
}

\setcitestyle{authoryear,open={(},close={)}}
\hypersetup{
    colorlinks=true,
    linkcolor=magenta,
    citecolor =cyan,
    filecolor=magenta,      
    urlcolor=magenta,
}

\usepackage{MnSymbol,wasysym}
\usepackage{todonotes}
\usepackage{amsfonts}
\usepackage{dsfont}
\usepackage{algorithm}
\usepackage{algorithmic}
\usepackage{comment}

\newtheorem{lemma}{Lemma}
\newtheorem{theorem}{Theorem}

\newtheorem{definition}{Definition}
\newtheorem{corollary}{Corollary}
\newtheorem{remark}{Remark}

\def\Rr{\mathbb{R}}
\def\Nn{\mathbb{N}}
\def\E{\mathbb{E}}
\def\bE{\mathbb{E}}
\def\bP{\mathbb{P}}

\newcommand{\ind}{\mathds{1}}

\def\kl{\mathrm{kl}}
\def\Alt{\mathrm{Alt}}

\def\Ec{\mathcal{E}}
\def\Hc{\mathcal{H}}

\def\Ic{\mathcal{I}}

\def\Sc{\mathcal{S}}

\def\cA{\mathcal{A}}

\def\cE{\mathcal{E}}
\def\cH{\mathcal{H}}

\def\ie{\textit{i}.\textit{e}.,~ }
\def\eg{\textit{e}.\textit{g}.,~ }

\def\oracle{\widetilde{\mathcal{P}}^\star}
\def\Deltat{\widetilde{\Delta}}

\begin{document}

\maketitle

\begin{abstract}
	This paper introduces a general multi-agent bandit model in which each agent is facing a finite set of arms and may communicate with other agents through a central controller in order to identify --in pure exploration-- or play --in regret minimization-- its optimal arm. The twist is that the optimal arm for each agent is the arm with largest expected \emph{mixed} reward, where the mixed reward of an arm is a weighted sum of the rewards of this arm for \emph{all agents}. This makes communication between agents often necessary. This general setting allows to recover and extend several recent models 
	for 
collaborative bandit learning, including the recently proposed federated learning with personalization \citep{shi2021federated}. In this paper, we provide new lower bounds on the sample complexity of pure exploration and on the
	regret. 
	We then propose a near-optimal algorithm for pure exploration.
	This algorithm is based on phased elimination with two novel ingredients: a data-dependent sampling scheme within each phase, aimed at matching a relaxation of the lower bound. 
\end{abstract}

\section{Introduction}\label{section:introduction}

Collaborative 
learning is a general machine learning paradigm in which a group of agents collectively train a learning algorithm. 
Some recent works have investigated how agents can efficiently perform sequential decision making in a collaborative context~\cite{zhu2021federated}. In particular, Shi et al. propose an interesting setting to tackle collaborative 
bandit learning when some level of personalization is required  \cite{shi2021federated}. Personalization leads to the twist that each agent should play the best arm in a \emph{mixed model} which is obtained as a combination of her \emph{local model} with the local model of other agents. In this work, we introduce a more general model retaining this idea, that we call the \textcolor{black}{weighted collaborative bandit model}.

In this model, there are $M$ agents and a finite number of $K$ arms. When agent $m$ samples arm $k$ at time $t$, she gets to observe a \emph{local reward} $X_{k,m}(t)\;,$ which is drawn from a $1$-sub-Gaussian~\footnote{\textcolor{black}{A random variable $X$ is said to be $\sigma^2$-sub-Gaussian if, \textcolor{black}{for any} $
\lambda\in\Rr$, $\ln (\E[ e^{\lambda(X-\E [X])}])\le \sigma^2\lambda^2/2$}.} distribution of mean $\mu_{k,m}$, independently from past observations and from other agents' observations. However, this agent does not necessarily seek to maximize her local reward, but rather some notion of \emph{mixed reward}, related to the utility of that arm for other agents. 
More specifically, we assume that agents share a \emph{known} weight matrix  $W := (w_{n,m})_{n,m} \in [0,1]^{M\times M}\;,$ such that $\sum_{n \in [M]} w_{n,m}= 1$ for all $m\in [M]\;,$ where $[n] := \{1,2,\dots,n\}\;.$ The mixed reward at time $t$ for agent $m$ and arm $k$ is defined as a weighted average of the local rewards across all agents $X'_{k,m}(t) := \sum_{n=1}^{M} w_{n,m}X_{k,n}(t)\;,$ and its expectation, called the expected mixed reward, is 
\begin{equation*}
\mu'_{k,m} :=  \sum_{n=1}^{M} w_{n,m}\mu_{k,n}\;.
\end{equation*}
We denote by $k_{m}^{\star} := \arg\max_{k \in [K]} \mu'_{k,m}$ the arm with largest expected mixed reward for agent $m$, assumed unique. Besides the degenerated case in which $w_{n,m} = \ind{(n = m)}$, in which each agent is solving their own bandit problem in isolation, the agents need to \emph{communicate}~; \ie to share information about their local rewards to other agents for everyone to be able to estimate their expected mixed rewards. A strategy for an agent is defined as follows: at each time $t$, each agent $m$ samples an arm $\pi_m(t)$, based on the available information, and then observes a noisy local reward from this arm. She has also the option to communicate information (\eg empirical means of past local observations) to a central controller (or server), which will broadcast this information to all other agents. Just like in any multi-armed bandit model, several objectives may be considered, either related to maximizing (mixed) rewards --or, equivalently,  minimizing regret-- or identifying the best arms, while maintaining a reasonably small communication cost.

\textcolor{black}{The weighted collaborative bandit model encompasses different frameworks previously studied in the literature. Notably, the paper
\cite{shi2021federated} studies a special case in which, given a level of personalization $\alpha \in [0,1]$, the mixed reward is an interpolation between $\mu_{k,m}$ and the average of local rewards $\tfrac{1}{M}\sum_{n=1}^{M}\mu_{k,n}$
which amounts to choosing $w_{m,m} = \alpha + \tfrac{1-\alpha}{M}$ and $w_{n,m} = \tfrac{1-\alpha}{M}$ for $n\neq m\;.$ The authors consider the objective of minimizing the regret 
while minimizing \textcolor{black}{for} the number of communication rounds, under the name federated multi-armed bandit with personalization\footnote{\textcolor{black}{Their setting, as well as ours, is neglecting some challenges typically addressed in (centralized) federated learning, such as privacy issues or dealing with communication interruption, this is why we prefer naming our framework "collaborative learning".}}.
In this paper, we mainly focus on the counterpart pure exploration problem in which agents should collaboratively identify their \textit{own} optimal arm in terms of expected mixed reward, with high confidence, and using as few exploration rounds as possible. This extends the well-studied fixed-confidence best arm identification problem~\citep{EvenDaral06,gabillon2012best,Jamiesonal14LILUCB} to the weighted collaborative
bandit setting.}


Another related setting is collaborative pure exploration~\citep{hillel2013distributed,tao2019collaborative,chen2021cooperative,zhu2021decentralized}, which considers $M$ agents solving the \textit{same} best arm 
identification problem. Most of these papers propose algorithms to solve this problem, while~\cite{tao2019collaborative} and~\cite{karpov2020collaborative} also prove lower bounds on sample complexity for the fixed-confidence and the fixed-budget best arm identification. 
Unlike our framework, ~\cite{chen2021cooperative} consider asynchronous agents, which can only sample at some times, whereas ~\cite{zhu2021decentralized,zhu2021federated} consider agents which can only communicate to some of the other agents. The goal of collaborative pure exploration is to reduce sample complexity at the cost of some communication rounds. Our model recovers the synchronous setting when considering $\mu_{k,m}=\mu_{k,m'}$ for all arm $k$ and agents $m,m'$ and $W=Id_{M}\;.$ 

\textcolor{black}{Besides collaborative learning, the work of \citep{wu2016contextual} considers a similar weighted model but in a different, contextual bandit setting in which a central controller chooses in each round an
arm for the unique agent (corresponding to a sub-population) that arrives, and aim at minimizing regret. 
Finally, the paper \cite{russac2021b} considers a pure exploration task in which the value of an arm is the weighted average of its utility for $M$ distinct populations (agents). In this setting $w_{n,m} = \alpha_n$, so that all agents have a common best arm, but the proposed algorithms do not aim at a low communication cost.}

\textcolor{black}{Collaborative learning}
in our general \textcolor{black}{weighted model} is also interesting beyond these examples. In particular, the work of~\cite{shi2021federated} mention\textcolor{black}{ed} possible applications to recommendation systems, for which one may want to go beyond uniformly personalized 
learning. The ``personalization'' part means that we favour the local rewards of agent $m$ over other agents' observations in the identification of optimal arm $k^\star_m$. But the introduction of a general weight matrix $W$ in our framework allows \textit{any} agent $m$ to consider \textit{any} linear combination of the other agents' observations, and then, different degrees of personalization across agents.

Such a setting could also be appropriate for adaptive clinical trials on $K$ therapies, run by $M$ teams who have access to different sub-populations of patients. \textcolor{black}{
In this context, each sub-population is typically aiming at finding their (local) best treatment. However,  solving their best arm identification in isolation may have a large sample complexity. If one is willing to assume that we have a weight matrix $W$ for which the best mixed arm of each agent coincides with its local best arm, \textit{i}.\textit{e}. $k^\star_m = \arg\max_{k \in [K]} \mu_{k,m}$ for any agent $m$, then solving best arm identification in the weighted collaborative bandit could have a much smaller sample complexity due to the sharing of information, while allowing the different clinical centers to communicate only once in a while. A first possibility to build such a weight matrix is to rely on a clustering of the different sub-populations so that the $\mu_{k,m}$ is supposed to be close to $\mu_{k,m’}$ (but not necessarily equal) for all $k$ and all agents $m,m’$ in the same cluster. Denoting by $C_m$ the cluster to which agent $m$ belongs, by setting $w_{m,n} = \mathds{1}(C_m = C_n)/|C_m|$ we would have the mixed mean of each agent be very close to their local means. Another possibility is to rely on a similarity function $S$ between sub-populations and define for all $n,m$ $w_{m,n}$ to be proportional to $S(m,n)$. This similarity function could be obtained prior to the learning phase by computing the similarities between subpopulation biomarkers, for instance.
}

\paragraph{Related work} Collaborative bandit learning has recently sparked wide interest in the multi-armed bandit literature. While some works do not deal with personalization~\citep{dubey2020differentially,shi2021almost,mitra2021exploiting}, others have for instance studied the integration of arm features in a modified 
setting with personalization~\citep{huang2021federated}.
\textcolor{black}{An interesting kernelized collaborative pure exploration problem was recently studied by \cite{du2021collaborative}. In their model, both agents and arms are described by feature vectors, and there is a known kernel encoding the similarity between the mean reward of each (agent,arm) pair. This independent work follows a similar approach as ours and also propose a near-optimal phased elimination algorithm inspired by a lower bound, but the models and related lower bounds are significantly different.} 



In 
bandits working in collaboration, the need for a small communication cost \textcolor{black}{indeed} makes algorithms based on \emph{phased eliminations} appealing. In such algorithms, agent(s) maintain a set of a
\textit{active} arms
that are candidate for being optimal, and potentially eliminate arms from this set at the end of each sampling phase. Adaptivity 
to the observed rewards (and, in our case, communication) is only needed between sampling phases, which are typically long. This type of structure has been used in various bandit settings, both for regret minimization or pure exploration objectives~\citep{auer2010ucb,hillel2013distributed,chen2017nearly,fiez2019sequential,shi2021federated,BKMP20}.
In some of these algorithms, including the one in~\cite{shi2021federated} which motivates this paper, the number of samples gathered from an arm which is active in some phase $r$ is fixed in advance. We believe that going beyond such a deterministic sampling scheme is crucial to achieve optimal performance with phased algorithms.
In order to achieve (near)-optimality, other phased algorithms rely on computing an oracle allocation from the optimization problem associated with a lower bound on the sample complexity~\citep{fiez2019sequential,du2021collaborative}. \textcolor{black}{In these works, based on this allocation, a total number of samples to collect in the current phase is computed, which depends on the identity of the surviving arms. The distribution of samples across arms for the current round is proportional to the oracle allocation, and is obtained through a rounding procedure. 
Compared to these works, our algorithm will show three distinctive features: an allocation inspired by a \emph{relaxation} of the lower bound, which does not \textcolor{black}{only} depend on the identity of surviving arms, 
and an alternative to the rounding procedure.} 


\paragraph{Contributions} \textcolor{black}{The authors of \cite{shi2021federated} exhibit an algorithm using phased elimination for federated bandit learning with personalization}, and prove a logarithmic regret bound, whose dependency in the parameters of the problem is conjectured to be sub-optimal. They also propose a heuristic improvement, based on a more adaptive exploration within each phase. 
 In this work, we take a step further in identifying the problem-dependent complexity of bandit learning in our novel \textcolor{black}{weighted collaborative bandit model  --which includes the setting  of ~\cite{shi2021federated} as a special case--} both from the pure exploration and the regret perspective. We propose new information theoretic lower bounds, and a recipe to design algorithms (nearly) matching those for pure exploration. Our main algorithmic contribution is 
 for \textcolor{black}{weighted} collaborative 
 best arm identification. Our phased elimination-based algorithm 
 achieves minimal exploration cost up to some logarithmic multiplicative factors, while using a constant amount of communication rounds. It relies on a \textcolor{black}{novel} data-dependent sampling scheme, \textcolor{black}{which renders its analysis trickier}. 
 The structure of the algorithm can easily be extended to Top-$N$ identification, that is, where each agent has to identify her own $N$ best arms (instead of her best for $N=1$) with respect to mixed rewards. We further compare our novel regret lower bound to what was conjectured in~\cite{shi2021federated} for the particular case of \textcolor{black}{federated learning with personalization.}

\paragraph{Notation} For both best arm identification and regret minimization, our complexity terms feature the (mixed) gaps of each agent $m$ and arm $k$, defined by 
\[\Delta'_{k,m} := \left\{\begin{array}{cl}
       \mu'_{k^\star_m,m}-\mu'_{k,m} & \text{if } k\neq k_{m}^\star\;,\\
       \min_{k \neq k_m^\star} \Delta_{k,m}' & \text{otherwise}\;.
                          \end{array}
\right.\]


\section{\textcolor{black}{Collaborative} 
Best-Arm Identification}\label{sec:federated_bai}


The goal of \textcolor{black}{collaborative} 
best-arm identification is that each agent $m$ identifies its optimal arm $k_m^\star$ by sampling the arms as little as possible and with few communication rounds.  Formally, a \textcolor{black}{collaborative} 
Best Arm Identification (BAI) algorithm consists of a sampling rule $\pi_m$ for each agent $m$, such that, at time $t$, either arm $\pi_m(t) \in [K]$ is sampled by $m$, or $\pi_m(t)=0$~; in that case, instead of picking an arm at time $t$, we allow the agent to remain idle, and not to select an arm.~\footnote{Note that, in a regret setting, an idle agent may exploit its empirical best arm.} Similarly to the communication model studied in~\cite{tao2019collaborative,shi2021federated}, communication only happens at the end of local sampling rounds for all agents, when all agents are idle at the same time. Besides the sampling rule,  
the BAI algorithm uses a stopping rule $\tau$ which determines when exploration is over \emph{for all agents}. The end of exploration is decided by the central server. Then, at time $\tau$, 
each player outputs a guess for its optimal arm with respect to mixed rewards, denoted by $\hat{k}_m$.

Our goal is to construct a $\delta$-correct strategy $\cA=(\pi,\tau,\hat{k})$, which satisfies, for any model $\mu \in \Rr^{K \times M}$, 
\[\mathbb{P}^{\cA}_{\mu}\left( \forall m \in [M],  \hat{k}_m  = k^{\star}_{m} \right) \geq 1-\delta\;,\]
while achieving a small  \emph{exploration cost} (e.g. in high probability or in expectation)
\begin{eqnarray*}
\text{Exp}_{\mu}(\cA) := \sum_{m=1}^{M}\sum_{k=1}^{K} N_{k,m}(\tau)\;,
\label{eq:exploration_cost}\end{eqnarray*}
where $N_{k,m}(t) := \sum_{s=1}^{t} \ind_{(\pi_m(s) = k)}$ is the number of selections of arm $k$ by agent $m$ up to time $t$, and a small \emph{communication cost}, defined as 
\[\text{Com}_{\mu}(\cA) := \sum_{t=1}^{\tau} \ind{\left(\Ic_{t}\right)}\;,\]
where $\Ic_t$ is the event that some information is shared between agents at round $t$.  


In our setting, we do not put constraints on the type of information that is exchanged in each communication round --which can be interesting when we consider privacy issues~\citep{dubey2020differentially,zhu2021federated}-- nor on the lengths of the messages. Each communication round has a unit cost. In a communication round, all agents send messages to the central server (\eg estimates of their local means) and the server can send back arbitrary quantities or instructions (\eg how many times each arm should be sampled in the next exploration phase, and when to communicate next). 

Moreover, contrary to the works of~\cite{hillel2013distributed,tao2019collaborative} on collaborative learning, we do not look at strategies \textcolor{black}{explicitly} minimizing for the number of communication rounds. Instead, our approach consists in proving a lower bound on the smallest possible exploration cost of a $\delta$-correct algorithm which would communicate at every round~; and then, finding an algorithm \textcolor{black}{for} which exploration cost matches this lower bound, while suffering a reasonable communication cost. 


\section{\textcolor{black}{Lower Bound}}\label{sec:lower_bound_bai}

We prove the following lower bound on the exploration cost of an algorithm in which all agents communicate to the central server their latest observation as soon as they received it. 
It holds for Gaussian rewards with variance $\sigma^2=1\;,$ meaning that the reward from an arm $k$ observed at time $t$ by agent $m$ will be $X_{k,m}(t) = \mu_{k} + \varepsilon_t\;,$ where $\varepsilon_t \sim \mathcal{N}(0,1)\;.$ We further assume that the weight matrix $W$ satisfies $w_{m,m} \neq 0$ for any agent $m \in [M]\;.$ 


\begin{theorem}\label{thm:LBBAI} Let $\mu$ be a fixed matrix of means in $\Rr^{K \times M}$. For any $\delta \in (0,1/2]$, let $\cA$ be a $\delta$-correct algorithm under which each agent communicates each reward to the central server after it is observed, and let us denote for any $k \in [K]\;,$ $m \in [M]\;,$ $\tau_{k,m} := \bE_{\mu}^{\cA}\left[ N_{k,m}(\tau) \right]\;,$ where $\tau$ is the stopping time. For any $m \in [M]$ and $k\neq k_m^\star$,  it holds that  
	\[\sum_{n}w_{n,m}^2\left(\frac{1}{\tau_{k,n}} + \frac{1}{\tau_{k^\star_m,n}}\right) \leq \frac{\left(\Delta'_{k,m}\right)^2/2}{\log(1/(2.4\delta))} \;,\]
and therefore $\bE_{\mu}\left[\text{Exp}_{\mu}(\cA)\right] \geq T^\star_{W}(\mu\textcolor{black}{'}) \log\left(\frac{1}{2.4\delta}\right)$, where 

	\begin{equation*}
		\begin{split}
		T_{W}^\star(\mu) := & \min_{t \in (\Rr^+)^{K \times M}} \left\{\sum_{(k,m) \in [K] \times [M]} t_{k,m}  : \forall m, k\neq k^\star_m, \sum_{n \in [M]}w_{n,m}^2\left(\frac{1}{t_{k,n}} + \frac{1}{t_{k^\star_m,n}}\right) \leq \frac{\left(\Delta'_{k,m}\right)^2}{2}\right\} \;.\\ 
		\end{split}
	\end{equation*}
\end{theorem}


The proof, given in Appendix~\ref{subapp:bai_lower_bound}, uses standard change-of-distribution arguments, together with classical results from constrained optimization. Note that, for $M=1$, we recover the complexity of best arm identification in a Gaussian bandit model~\citep{garivier2016optimal}. 

\textcolor{black}{
\paragraph{Computing the complexity term}
The optimization problem which defines $T^\star_{W}(\mu)$ belongs to the family of disciplined convex optimization problems, and can  be numerically solved using available solvers, such as CVXPY~\citep{agrawal2018rewriting,diamond2016cvxpy}. We now illustrate, on a 
\textcolor{black}{small} example, the possible reduction in exploration cost that can be obtained by solving weighted collaborative best arm identification instead of 
\textcolor{black}{$M$ parallel best arm identification problems}. We consider $K=M=2$ and a similarity $S(1,2) = 0.9$ between the two agents, which 
yields the \textcolor{black}{following normalized} weight matrix
\[W = \frac{1}{1.9}\begin{bmatrix} 1 & 0.9\\ 0.9 & 1  \end{bmatrix}\:. \]
\textcolor{black}{Considering the following} matrix of 
\textcolor{black}{expected rewards}
\[ \mu = \begin{bmatrix} 0.9 & 0.8\\ 0.1 & 0.5 \end{bmatrix} \;,\]
for which arm 1 is the local best arm for both agents and is also their best mixed arm, we obtain $T^\star_{W}(\mu) \approx 28$. However, if each agent solves its own best arm identification problem in isolation (which amounts to using $W = \text{Id}_2$), the resulting exploration cost scales with   $T^\star_{\text{Id}_2}(\mu) \approx 101 > 3T^\star_{W}(\mu)$.
}

\paragraph{From lower bounds to algorithms} In single-agent pure exploration tasks, lower bounds are usually guidelines to design optimal algorithms, as they allow to recover an oracle allocation (\ie the $\arg\min$ for $t\in (\Rr^+)^{K,M}$ in the definition of $T_{W}^\star(\mu)$) which algorithms can try to achieve by using some tracking~\citep{garivier2016optimal,du2021collaborative,russac2021b}\textcolor{black}{. Yet, these approaches may be computationally expensive, as they solve the optimization problem featured in the lower bound in every round.}


In the next section, we will propose an alternative approach for our \textcolor{black}{collaborative} 
setting, which exploits the knowledge of the lower bound within a phased elimination algorithm. This is crucial to maintain a small communication cost \textcolor{black}{and also permit to reduce the computational complexity compared to a pure tracking approach}. Our algorithm will rely on a \textit{relaxed} complexity term $\widetilde{T}^\star_W(\mu)$, which is within constant factors of $T^\star_W(\mu)$, as proved in Appendix~\ref{subapp:proofs_federated_bai}.

\begin{lemma} \label{lem:ttilde_optimal} Introducing the quantity
	\begin{equation*}
		\begin{split}
			\widetilde{T}_W^\star(\mu) := & \min_{t \in (\Rr^+)^{K \times M}} \left\{\sum_{(k,m) \in [K] \times [M]} t_{k,m}  \ : \ \forall m, \forall k \in [K], \sum_{n \in [M]}\frac{w_{n,m}^2}{t_{k,n}} \leq \frac{\left(\Delta'_{k,m}\right)^2}{2}\right\} \;,\\
		\end{split}
	\end{equation*}
it holds that $\widetilde{T}_W^\star(\mu) \leq T_W^\star(\mu) \leq 2 \widetilde{T}_W^\star(\mu)$.
\end{lemma}

\textcolor{black}{
Compared to $T_W^\star(\mu)$, a nice feature of $\widetilde{T}_W^\star(\mu)$ is that its constraint set does not depend on the knowledge of $(k_m^\star)_{m \in [M]}$, which will allow us to design algorithms that do not suffer too much from bad empirical guesses for $k_m^\star$ in early phases. We further remark that the computation of $\widetilde{T}_W^\star(\mu)$ and that of its associated oracle allocation (see Definition~\ref{def:oracle}) are slightly easier than for $T_W^\star(\mu)$. Indeed, the optimization problem which defines $\widetilde{T}_{W}^\star(\mu)$ 
can be decoupled across arms. Computing for every arm $k \in [K]\;$ the vector
	\[ \widetilde{\tau}^{k} = \underset{{\tau^k \in (\Rr^+)^{M}}}{\text{argmin}} \ \sum_{(k,m) \in [K] \times [M]} \tau^k_{m} \text{ s.t. } \forall m, \sum_{n \in [M]}\frac{w_{n,m}^2}{\tau^k_{n}} \leq \frac{\left(\Delta'_{k,m}\right)^2}{2} \;,\]
we obtain the argmin in $\widetilde{T}_W^\star(\mu)$ by setting $(t_{k,m})_{k,m} = (\widetilde\tau^k_m)_{k,m}$.
The computation of $\widetilde{T}_{W}^\star(\mu)$ can therefore be done by solving $K$ disciplined optimization problems (\textit{e}.\textit{g}., with CVXPY~\citep{agrawal2018rewriting,diamond2016cvxpy}) involving $M$ variables, instead of one optimization problem with $K \times M$ variables. 
}


\begin{definition}\label{def:oracle}
For any $\Delta \in (\Rr^+)^{K \times M}$, the oracle $\oracle(\Delta)$ is
	\[\underset{(\tau_{k,m})_{k,m} \in (\Rr^+)^{K \times M}}{\arg\min} \ \sum_{k,m} \tau_{k,m} \text{ s.t.}  \forall m \in [M], \forall k \in [K], \ \sum_{n \in [M]}\frac{w_{n,m}^2}{\tau_{k,n}} \leq \frac{\left(\Delta_{k,m}\right)^2}{2}\;.\]
\end{definition}

	With this notation, observe that $\widetilde{T}^\star(\mu\textcolor{black}{'}) = \sum_{k,m} \tau_{k,m}\;,$ where $(\tau_{k,m})_{k,m} \in \oracle\left(\Delta'\right)$. 
	The following lemma will be useful to compare values from different oracle problems. The full lemma along with its proof is available in Appendix~\ref{subapp:proofs_federated_bai}. 


\begin{lemma}\label{lem:12}
	Consider $\Delta\;,$ $\Delta' \in (\Rr^+)^{K \times M}$, such that $\tau \in \oracle(\Delta)$ and $\tau' \in \oracle(\Delta')$. Moreover, assume that there is a positive constant $\beta$ such that: $\forall k,m, \Delta'_{k,m} \leq \beta \Delta_{k,m}$. Then
	\[ \frac{1}{\beta^2} \sum_{k,m} \tau_{k,m} \leq \sum_{k,m} \tau'_{k,m}\;.\]
\end{lemma}

\section{A Near-Optimal Algorithm For Best Arm Identification}\label{sec:federated_bai_algo}

We now introduce an algorithm for \textcolor{black}{collaborative} 
best-arm identification, called \textcolor{black}{W-CPE}
-BAI for Weighted \textcolor{black}{Collaborative} 
Phased Elimination, stated as Algorithm~\ref{algo:fpe_bai}. To present an analysis of this algorithm, we assume that, for any $k,m \in [K] \times [M]\;,$ $\mu_{k,m} \in [0,1]\;,$ 
\textcolor{black}{and that} 
local rewards $(X_{k,m}(t))_{k,m,t}$ are $1$-sub-Gaussian.

\textcolor{black}{W-CPE}
-BAI proceeds in phases, indexed by $r\;.$ In phase $r\;,$ we let $B_m(r)$ be the set of active arms for agent $m\;,$ and $B(r) = \cup_{m} B_m(r)$ be the set of arms that are active for at least one agent. The algorithm maintains proxies for the gaps  $(\Deltat_{k,m}(r))_{k \in [K],m \in [K]}$ that are halved 
at the end of each phase for arms that remain active. At the beginning of each round, the oracle allocation $t(r)$, with respect to the proxy gaps, is computed, as well as the number of new samples $d_{k,m}(r)$ that player $m$ should get from arm $k$ in phase $r\;.$ $d_{k,m}(r)$ is defined such that the total number of selections of arm $k$ by agent $m$ becomes close to (a quantity slightly larger than) $t_{k,m}(r)\log(1/\delta)\;.$ We observe that any arm $k\notin B(r)$ will not get any new samples in phase $r\;,$ as the proxy gaps $(\Deltat_{k,n}(r))_n$ are identical to those in the previous phase~; therefore $t_{k,n}(r) = t_{k,n}(r-1)\;.$ \textcolor{black}{In contrast to prior works, where the allocation in each round only depends on the identity of the surviving arms and the round index~\citep{fiez2019sequential,du2021collaborative}, in \textcolor{black}{W-CPE}
-BAI \textcolor{black}{it also depends on when the arms have been eliminated (which condition the value of their frozen \textcolor{black}{proxy} gaps)}.}

After each agent $m$ samples arm $k$ $d_{k,m}(r)$ times, they send their local means $\hat{\mu}_{k,m}(r)$ to the central server, which computes the mixed mean estimates $\hat{\mu}'_{k,n}(r) := \sum_{m=1}^{M}w_{n,m} \hat{\mu}_{k,m}(r)\;.$ The active sets $(B_m(r))_{m \in [M]}$ of all agents are then updated by removing arms whose mixed means are too small. As in several prior works~\citep{kaufmann2013information,shi2021federated}, we rely on confidence intervals to perform these eliminations. However, constructing confidence intervals on the mixed means --which are linear combinations of the local means-- under our adaptive sampling rule is more challenging than when the number of samples from an active arm in phase $r$ is fixed in advance --which is the case for instance in the algorithm in~\cite{shi2021federated}. The width of our confidence intervals scales with the following quantity

\begin{definition}\label{def:omega}
	For any $k,m$, and round $r \geq 0$, we define
	\[\Omega_{k,m}(r) := \sqrt{\beta_{\delta}(n_{k,\cdot}(r))\sum_{n=1}^M \frac{w_{n,m}^2}{n_{k,n}(r)}} \;,\]
	where $n_{k,m}(r)$ is the number of times arm $k$ was selected by agent $m$ by the end of phase $r$ (included), and 
	$N \mapsto \beta^{\delta}(N)$ is a threshold function defined for any $N \in (\Rr^+)^{M}\;.$
\end{definition}

\begin{algorithm*}[tb]
\caption{\textcolor{black}{Weighted Collaborative} 
Phased Elimination for Best Arm Identification (\textcolor{black}{W-CPE}
-BAI)}
   \label{algo:fpe_bai}
\begin{algorithmic}
   \STATE {\bfseries Input:} $\delta \in (0,1)$, $M$ agents, $K$ arms, weights matrix $W$
   \STATE Initialize $r \gets 0$, $\forall k,m, \widetilde{\Delta}_{k,m}(0) \gets 1, n_{k,m}(0) \gets 1$, $\forall m, B_m(0) \gets [K]$
   \STATE Draw each arm $k$ by each agent $m$ once
   \REPEAT
	\STATE \textcolor{gray}{\textcolor{black}{\# Central server}}
   \STATE $B(r) \gets \bigcup_{m \in [M]} B_m(r)$
	\STATE Compute $t(r) \gets \oracle\left(\left(\sqrt{2} \Deltat_{k,m}(r)\right)_{k,m}\right)$ 
	\STATE For any $k \in [K]$, compute \[(d_{k,m}(r))_{m\in [M]} \gets \arg\min_{d \in \Nn^{M}} \sum_{m} d_{m} \text{ s.t. } \forall m \in [M], \frac{n_{k,m}(r-1)+d_{m}}{\beta_{\delta}(n_{k,\cdot}(r-1)+d)} \geq t_{k,m}(r)\]
	\STATE \textcolor{black}{Send to each agent $m$ $(d_{k,m}(r))_{k,m}$ and $d_{\max} := \max_{n \in [M]} \sum_{k \in [K]} d_{k,n}(r)$}
     \STATE
	\STATE \textcolor{gray}{\textcolor{black}{\# Agent $m$}}
	\STATE Sample arm $k \in B(r)$ $d_{k,m}(r)$ times, so that $n_{k,m}(r) = n_{k,m}(r-1)+d_{k,m}(r)$
	\STATE \textcolor{black}{Remain idle for $d_{\max}-\sum_{k \in [K]} d_{k,m}(r)$ rounds}
	\STATE \textcolor{black}{Send to the server empirical mean $\hat{\mu}_{k,m}(r):=\sum_{s \leq n_{k,m}(r)} X_{k,m}(s)/n_{k,m}(r)$ for any $k \in [K]$}
    \STATE
	\STATE \textcolor{gray}{\textcolor{black}{\# Central server}}
   \STATE Compute the empirical mixed means $(\hat{\mu}'_{k,m}(r))_{k,m}$ based on $(\hat{\mu}_{k,m}(r))_{k,m}$ and $W$
   \STATE {\textcolor{gray}{// \emph{Update set of candidate best arms for each user}}}
   \FOR{$m=1$ {\bfseries to} $M$} 
   \STATE \[B_{m}(r+1) \gets \left\{ k \in B_{m}(r) \mid \hat{\mu}'_{k,m}(r) + \Omega_{k,m}(r) \geq \max_{j \in B_m(r)} \left( \hat{\mu}'_{j,m}(r) - \Omega_{j,m}(r) \right) \right\}\]
   \ENDFOR
   \STATE {\textcolor{gray}{// \emph{Update the gap estimates}}}
   \STATE For all $k,m$, $\Deltat_{k,m}(r+1) \gets \Deltat_{k,m}(r) \times (1/2)^{\ind \left(k \in B_m(r+1) \land |B_m(r+1)|>1 \right)}$
   \STATE $r \gets r+1$
   \UNTIL{$\forall m \in [M], |B_m(r)|\leq1$}
	\STATE {\bfseries Output:} $\left\{ k \in B_m(r) : m \in [M] \right\}$
\end{algorithmic}
\end{algorithm*}

Leveraging some recent time-uniform concentration inequalities \cite{kaufmann2018mixture}, we exhibit below a choice of threshold that yields valid confidence intervals on the mixed means (by "projecting" confidence intervals that can be obtained on local means, see Proposition $24$ in \cite{kaufmann2018mixture}). \textcolor{black}{The fact that the confidence interval depends on the random number of past draws (and not just the index of the round) leads to some non-trivial complication in the analysis, 
with the
introduction of quantities $(d_{k,m})_{k,m}$.} The proof is given in Appendix~\ref{subapp:proofs_federated_bai_algo}, where we also provide an explicit expression of the function $g_M\;.$ 


\begin{lemma}\label{lem:validCI}
Let us define 
	\[\beta_{\delta}(N) := 2\left( g_M\left(\frac{\delta}{KM}\right) + 2 \sum_{m=1}^{M} \ln(4+\ln(N_{m})) \right)\;,\]
	for any $N \in (\Nn^*)^{M}$, where $g_{M}$ is some non-explicit function, defined in~\cite{kaufmann2018mixture}, that satisfies $g_{M}(\delta) \simeq \log\left(\frac{1}{\delta}\right) + M\log\log\left(\frac{1}{\delta}\right)$. Then the good event \[\mathcal{E} := \left\{\forall r \in \mathbb{N}, \forall m, \forall k, \left|\hat{\mu}_{k,m}'(r) - \mu_{k,m}'\right| \leq \Omega_{k,m}(r)\right\}\;\]
holds with probability larger than $1-\delta$.
\end{lemma}





From this lemma, it easily follows that \textcolor{black}{W-CPE}
-BAI is $\delta$-correct for the above choice of threshold function, as, for any agent $m\;,$ no good arm $k_m^\star$ can ever be eliminated from $B_m(r)$ at round $r\;.$ The fact that the sample complexity of \textcolor{black}{W-CPE}
-BAI scales with $\widetilde{T}^\star(\mu\textcolor{black}{'})$ on the good event $\cE$ comes from the interplay between the expression of $\Omega_{k,m}(r)$ (which, up to the threshold function, is exactly one of the constraints featured in the lower bound) together with the definition of the allocation $t(r)\;,$ which leads to the following crucial result in our analysis

\begin{lemma}\label{lem:upperbound_omega}
	On 
	$\cE$, $\forall k,m,r\geq0$, 
	$\Omega_{k,m}(r) \leq \Deltat_{k,m}(r)\;.$
\end{lemma}

\begin{proof}
	For any round $r$ and arm $k$, by Algorithm~\ref{algo:fpe_bai} and the definition of oracle $t_{k,\cdot}(r)$, for any agent $m$,{\small
	\begin{eqnarray*} \Omega_{k,m}(r) & = & \sqrt{\sum_n w_{n,m}^2 \frac{\beta_{\delta}(n_{k,\cdot}(r-1)+d_{k,\cdot})}{n_{k,n}(r-1)+d_{k,n}(r)}}
		\leq \left\{\begin{array}{cl}
       \sqrt{\sum_n  \frac{w_{n,m}^2}{t_{k,n}(r)}} \leq \Deltat_{k,m}(r) & \text{if } k \in B(r)\;,\\
       \sqrt{\sum_n \frac{w_{n,m}^2}{t_{k,n}(r'_k)}} \leq \Deltat_{k,m}(r'_k) = \Deltat_{k,m}(r) & \text{otherwise}\;,
                          \end{array}
    \right. 
	\end{eqnarray*}} 
	where when $k \not\in B(r)$, $r'_k := \sup \{ r' \geq 0 : k \in B(r') \}$, and we use the fact that $d_{k,m}(r)=0$ when $k \not\in B(r)$. 
\end{proof}

We did not put much emphasis on the way communications are performed between the agents and the central server, as several choices are possible. The important part is that the server receives all values of the local means $(\hat{\mu}_{k,m}(r))_{k,m}$ at the end of round $r\;.$ Our suggestion is that the central server maintains the sets $(B_{m}(r))_{m\in[M]}\;,$ calls the oracle, and sends to all agents their values of $(d_{k,m}(r+1))_{k,m}$ at the end of each phase $r$. In any case, the number of communication rounds in our definition will be equal to the number of phases used by \textcolor{black}{W-CPE}
-BAI. All in all, we prove 

\begin{theorem}\label{th:sample_complexity_bai} With probability $1-\delta$, \textcolor{black}{W-CPE}
-BAI outputs the optimal arm for each agent with an exploration cost at most
\begin{eqnarray*}
32\widetilde{T}_W^\star(\mu)\log_2({8}/{\Delta_{\min}'})\log\left(\frac{1}{\delta}\right) + o_{\delta}\left(\log\left(\frac{1}{\delta}\right)\right)\;,\end{eqnarray*} 
and at most $\left\lceil\log_2\left({8}/{\Delta_{\min}'}\right)\right\rceil$  communication rounds,
where $\Delta'_{\min} := \min_{k \in [K],m \in [M]} \Delta'_{k,m}\;.$
\end{theorem}

\paragraph{Proof sketch} The detailed proof is given in Appendix~\ref{app:proofUB}, where we also provide an explicit upper bound on the exploration cost. We let $R$ denote the (random) number of phases used by the algorithm before stopping. On the good event $\cE\;,$ we can prove that the algorithm never eliminates $k_m^\star$ therefore $R := \max_{m}\max_{k \neq k_m^{\star}} R_{k,m}$ where $R_{k,m}$ is the last phase in which $k \in B_m(r)\;.$ Using Lemma~\ref{lem:upperbound_omega}, we can easily establish that 
\[R_{k,m} \leq r_{k,m} := 
 \min \left\{r \geq 0 : 4 \times 2^{-r} < \Delta'_{k,m}\right\}\]
 which satisfies $r_{k,m}\leq \log_2(8/\Delta_{k,m}')$. This yields $R \leq \log_{2}(8/\Delta_{\min}')\;,$ and 
further permits to prove that the proxy gaps can be lower bounded by the true gaps: 
\[\forall r \leq R, \forall k \in [K], \forall m \in [M], \ \Deltat_{k,m}(r) \geq \frac{1}{8}\Delta'_{k,m}\;.\]
See Corollary~\ref{cor:lubounds_gap} in Appendix~\ref{app:proofUB}. Using the monotonicity properties of the oracle that are stated in Lemma~\ref{lem:12}, we can then establish that the allocation $t(r)$ computed from the proxy gaps in the algorithm satisfies 
\begin{equation}\label{eq:importStep}\forall r \leq R, \ \sum_{k \in [K],m \in [M]}t_{k,m}(r) \leq 32 \widetilde{T}_W^\star(\mu). \end{equation}
To upper bound the exploration cost, the next step is to relate $n_{k,m}(R)$ to the oracle allocations. To do so, we observe that if $R'_{k,m}$ is the last round before $R$ such that $d_{k,m}(r) \neq 0$ (i.e. the last round in which arm $k$ is actually sampled by agent $m$~; then $n_{k,m}(R)=n_{k,m}(R'_{k,m})$), we have by definition of the $(d_{k,m}(r))_{k,m,r}$ that 
\begin{eqnarray*}n_{k,m}(R'_{k,m}) &\leq& t_{k,m}(R'_{k,m}) \beta_{\delta}(n_{k,\cdot}(R'_{k,m})) + 1 \leq t_{k,m}(R'_{k,m}) \beta_{\delta}(n_{k,\cdot}(R)) + 1\;.
\end{eqnarray*}
See Lemma~\ref{lem:UBn} in Appendix~\ref{app:proofUB}. We can then upper bound $\tau := \sum_{k,m} n_{k,m}(R)$ as follows 
\begin{eqnarray*}\label{eq:final_inequality}
	\tau & \le & \sum_{k,m} t_{k,m}(r'_{k,m}) \beta_{\delta}(n_{k,\cdot}(R))+KM \leq \sum_{k,m}\sum_{r\leq R} t_{k,m}(r) \beta^*(\tau) + KM\;, \\
	\end{eqnarray*}
\begin{eqnarray}\label{eq:final_inequality}
	\text{and } \tau & \leq & R \times 32 \widetilde{T}_W^\star(\mu) \beta^*(\tau) + KM\;, \text{ where we use~\eqref{eq:importStep} and introduce}    
\end{eqnarray}
\[\beta^*(\tau) := 2\left(g_M\left(\frac{\delta}{KM}\right) + 2 M \ln \left(4 + \ln \left(\tau\right)\right)\right)\;.\]
The end of the proof consists in using the known upper bound on $R\;,$ and finding an upper bound for the largest $\tau$ satisfying the inequality in~\eqref{eq:final_inequality}. 
\qed
\paragraph{Discussion} Theorem~\ref{th:sample_complexity_bai} proves that \textcolor{black}{W-CPE}
-BAI is matching the exploration lower bound of Theorem~\ref{thm:LBBAI} in a regime where $\delta$ is small, up to multiplicative constants, including a logarithmic term in $1/\Delta_{\min}'$. It achieves this using only $\left\lceil \log_2\left( 8/\Delta'_{\min} \right) \right\rceil$ communication rounds.
\textcolor{black}{We note that a similar extra \textcolor{black}{multiplicative} logarithmic factor \textcolor{black}{is present} 
in the analysis of near-optimal phased algorithms in other contexts \cite{fiez2019sequential,du2021collaborative}. Such a quantity appears as an upper bound on the number of phases\textcolor{black}{,} and may be a price to pay for the phased structure.}




	\paragraph{On the communication cost} We argue that the communication cost of \textcolor{black}{W-CPE}
	-BAI is actually of the same order of magnitude as that featured in some related work. In~\cite{shi2021federated}, which is the closest setting to our framework, the equivalent number of communication rounds $p$ needed to solve the regret minimization problem is upper bounded by 
	$\mathcal{O}\left(2\log_2\left( {8}/{(\sqrt{M} \Delta'_{\min})} \right)\right)$. 
	In the setting of collaborative learning --where $M$ agents face the same set of arm distributions and $W=Id$-- ~\cite{hillel2013distributed} in their Theorem $4.1$ prove that an improvement of multiplicative factor $1/M$ on the exploration cost for a traditional best arm identification algorithm can be reached by using at most $\left\lceil \log_2(1/\Delta_{\min}) \right\rceil$ communication rounds\textcolor{black}{, where $\Delta_{\min}$ is the gap between the best and second best arms. 
	}

\paragraph{Experimental validation} We propose in Appendix~\ref{sec:experiments} an empirical evaluation of W-CPE-BAI for the weight matrix $w_{m,n} = \alpha \ind(n=m) + \frac{1-\alpha}{M}$ which corresponds to the setting studied by \cite{shi2021federated}. In this particular case, we propose as a baseline a counterpart of the regret algorithm of  \cite{shi2021federated} which we call PF-UCB-BAI. Our experiments on a synthetic instance show that W-CPE-BAI and PF-UCB-BAI have similar performances in terms of exploration cost and that W-CPE-BAI becomes better when the level of personalization $\alpha$ is smaller than $0.5$. Moreover, W-CPE-BAI uses less rounds of communication than PF-UCB-BAI for all values of $\alpha$. Finaly, the near-optimality of W-CPE-BAI is empirically observed when compared to an oracle algorithm which has access to the true gaps. 
\textcolor{black}{We refer the reader to Appendix~\ref{sec:experiments} for \textcolor{black}{further} details on the optimization libraries that were used.}

\begin{remark}
	The analysis and the structure of Algorithm~\ref{algo:fpe_bai} \textcolor{black}{have the potential to} be extended to other pure exploration problems, with similar guarantees. In Appendix~\ref{app:topN}, we illustrate this claim by tackling Top-$N$ identification.
\end{remark}



\section{Regret Lower Bound}\label{section:regret}

In contrast to the BAI setting,~\cite{shi2021federated}
considered the objective of minimizing the regret
\begin{eqnarray*}\mathcal{R}_{\mu}(T) &:=& \bE \left[\sum_{m=1}^{M}\sum_{t=1}^{T} \left(\mu'_{k_m^\star,m} - \mu_{\pi_m(t),m}\right)\right]\;. 
\end{eqnarray*}

They provided a conjecture on the lower bound on regret in personalized  
\textcolor{black}{federated} learning~\cite[see,][Conjecture~$1$]{shi2021federated}. 
As mentioned in introduction, their reward model is a special case of ours with weights
$w_{m,m} = \alpha + \tfrac{1-\alpha}{M}$ and $w_{n,m} = \tfrac{1-\alpha}{M}$ for $n\neq m\;.$ In this section, we prove a regret lower bound with general weights \textcolor{black}{ that proves in particular  Conjecture $1$ in~\cite{shi2021federated}}. 

We prove the following result, {for an algorithm that selects in each round exactly one arm for each agent}, and all agents communicate after each round.  This lower bound holds for Gaussian arms with variance $\sigma^2=1$. 

\begin{theorem}\label{thm:LBRegret} Any uniformly efficient algorithm\footnote{A uniformly efficient algorithm satisfies $\mathcal{R}_{\mu}(T) = o(T^\gamma)$ for any $\gamma \in (0,1)$ and any possible instance $\mu$.} in which all agents communicate after each round satisfies 
\[\liminf_{T\rightarrow \infty}\frac{\mathcal{R}(T)}{\log(T)} \geq C^\star_W(\mu)\;,\]
where 
\begin{equation*}
	\begin{split}
		C^\star_W(\mu) & := \min_{c = (c_{k,m})_{k,m : k_m^\star \neq k}}\left\{ \sum_{k=1}^{K}\sum_{m : k_{m}^\star \neq k} c_{k,m}\Delta_{k,m}' \ : \ \forall k \in [K], \forall m \in [M], \sum_{n : k_n^\star \neq k} \frac{w_{n,m}^2}{c_{k,n}} \leq \frac{(\Delta'_{k,m})^2}{2}\right\}\;.	\end{split}
\end{equation*}
We recall that for any agent $m$,  $\Delta'_{k_m^\star,m} := \min_{k' \neq k_m^\star} \Delta'_{k',m}\;.$ 
\end{theorem}

 Theorem~\ref{thm:LBRegret} may be viewed as an extension of the lower bound of~\cite{LaiRobbins85bandits} to the \textcolor{black}{collaborative} 
 setting. Its proof, given in Appendix~\ref{subapp:regret_lower_bound}, uses classical ingredients for regret lower bounds in  (single agent) structured bandit models~\citep{Combes17OSSB,GravesLai97}.

The generic approaches proposed  in~\cite{Combes17OSSB,degenne2020structure} for optimal regret minimization in structured bandits could therefore be useful. However, to turn them into a reasonable algorithm for the \textcolor{black}{collaborative} 
setting, we would need to preserve the phased elimination structure. Analogous to the relaxation in the BAI setting, we can also define a 
relaxed complexity term $\widetilde{C}_{M}(\mu)\;,$ which does not require the knowledge of 
the best arms, and is within constant factors of $C(\mu)$ by the following Lemma~\ref{lem:tilde_C_mu} (which proof is given in Appendix~\ref{subapp:proofs_regret}).

\begin{lemma}\label{lem:tilde_C_mu}
Introducing the quantity 
\begin{equation*}
	\begin{split}
		\widetilde{C}^\star_W(\mu) & := \min_{c \in (\Rr^+)^{K\times M}} \left\{ \sum_{k=1}^{K}\sum_{m=1}^M c_{k,m}\Delta_{k,m}' \text{ s.t. } \forall k \in [K], \forall m \in [M], \sum_{n =1}^{M} \frac{w_{n,m}^2}{c_{k,n}} \leq \frac{(\Delta'_{k,m})^2}{2}\;\right\},\\
	\end{split}
	\end{equation*}
it holds that $C^\star_{W}(\mu)\le\widetilde{C}^\star_{W}(\mu)\le 4C^\star_{W}(\mu)\;.$ 
\end{lemma}

\textcolor{black}{
The constrained optimization problems governing the regret lower bound are subtly different from those in the BAI setting. In regret minimization, unknown gaps $(\Delta'_{k,m})_{k,m}$ appear both in the constraints and in the objective. In contrast, in BAI, $(\Delta'_{k,m})_{k,m}$ only appear in the constraints. Due to this difference, designing a similar algorithm for regret minimization leads to an extra multiplicative $\mathcal{O}(1/\Delta'_{\min})$ factor in the upper bound. 
A very similar issue exists in 
a different
structured bandit problem,
bandits on graphs with side observations, where arms are connected through a graph with unweighted edges, and the agent receives the reward associated with the selected arm \textit{and its neighbors} at a given round. 
In~\cite[Problem P$1$]{buccapatnam2014stochastic}, authors describe a constrained linear optimization problem which governs the regret lower bound, and face the same issue of scaling in $\mathcal{O}(1/\Delta'_{\min})$. 
Dealing with this problem would be an interesting subsequent work.} 


\section{Conclusion}\label{sec:discussion}

This paper introduced a general framework for \textcolor{black}{collaborative} 
learning in multi-armed bandits. Our work presents two novel lower bounds: one on the exploration cost 
for pure exploration,
and another on cumulative regret. The latter permits to 
prove a  prior conjecture on the topic. Moreover, we propose a phased elimination algorithm for fixed-confidence best arm identification. 
This algorithm tracks the optimal allocation from the pure exploration lower bound \textcolor{black}{through a novel approach solving a relaxed optimization problem linked to the lower bound}. 
The exploration cost of this algorithm is matching the lower bound up to logarithmic factors. 
\textcolor{black}{This strategy can be extended to other pure exploration problems, such as Top-$N$ identification, as shown in Appendix~\ref{app:topN}. }

\textcolor{black}{As} \textcolor{black}{mentioned in introduction, our collaborative setting was motivated by the design of collaborative adaptive clinical trials for personalized drug recommendations, where several patient subpopulations (for instance, representing several subtypes of cancer) are considered and sequentially treated. 
However, in practice, especially when dealing with patient data, disclosing the mean response values to the central server should be handled with care to preserve the anonymity of the patients. A possible solution to overcome this problem would be to carefully combine our algorithm with a data privacy-preserving method, for instance by adding some noise to the data~\citep{dubey2020differentially}.} 

Our code and run traces are available in an open-source repository.~\footnote{\texttt{https://github.com/clreda/near-optimal-federated}}

\begin{ack}
Clémence Réda was supported by the French Ministry of Higher Education and Research [ENS.X19RDTME-SACLAY19-22]. The authors acknoweldge the support of the French National Research Agency under the project [ANR-19-CE23-0026-04] (BOLD).
\end{ack}

\bibliography{references}
\bibliographystyle{abbrvnat}

\newpage


\appendix

\section{Proof of the Lower Bounds} \label{app:LowerBounds}

\subsection{Lower Bound on the Exploration Cost: Proof of Theorem~\ref{thm:LBBAI}}\label{subapp:bai_lower_bound}

Let us denote, for any model $\mu \in \Rr^{K \times M}$ and agent $m\;,$ $k^\star_m := \arg\max_{k \in [K]} \mu'_{k,m}$ (which is assumed unique). Define the set of alternative models in $\Rr^{K \times M}$ with respect to $\mu\;:$ 
	\[ \mathrm{Alt}(\mu)  := \left\{\lambda : \exists m, \exists k \neq k^\star_m  : \lambda'_{k,m} > \lambda'_{k_m^\star,m} \right\} \;, \]
where $\lambda'_{k,m} := \sum_{n \in [M]} w_{n,m}\lambda_{k,n}$ for any arm $k$ and agent $m\;.$
	Assume that stopping time $\tau$ is almost surely finite under $\mu$ for algorithm $\cA$. Let event $\Ec_\mu := \left\{\exists m :  \hat{k}_m \neq k_m^\star\right\}\;.$ 
	As algorithm $\cA$ is $\delta$-correct, it holds that $\bP_{\mu}(\Ec_\mu) \leq \delta$ and $\bP_{\lambda}(\Ec_\mu) \geq 1 - \delta$ for all $\lambda \in \Alt(\mu)\;.$ As this event belongs to the filtration generated by all past observations up to the final stopping time $\tau$, using Theorem $1$ from~\cite{garivier2016optimal} and $\delta \leq 1/2\;,$ it holds that 
	\begin{equation}\frac{1}{2} \sum_{k,m} \tau_{k,m} (\mu_{k,m} - \lambda_{k,m})^2 
	\geq \log\left(\frac{1}{2.4\delta}\right)\;.\label{CD}\end{equation}
		We first prove that, 
		for any $k,m\;,$ 
		$\tau_{k,m} > 0\;.$ Indeed, if $w_{m,m} \neq 0\;,$ it is possible to pick $\lambda$ that only differs from $\mu$ by the entry $\lambda_{k,m}\;,$ in such a way that arm $k$ becomes optimal (or sub-optimal) for user $m\;.$ From Equation~\eqref{CD}, we get $\frac{1}{2}\tau_{k,m}(\mu_{k,m} - \lambda_{k,m})^2 >0$ and the conclusion follows. 

We now fix agent $m$ and $k \neq k_{m}^\star\;,$ and try to find a more informative alternative model $\lambda\;.$ We look for it in a family of alternative models under which only arms $k$ and $k_m^\star$ are modified, for \emph{all} agents, in order to make arm $k$ optimal for agent $m\;.$ Given two nonnegative sequences $(\delta_{n})_{n \in [M]}$ and $(\delta'_{n})_{n\in [M]}\;,$ we define $\lambda =(\lambda_{k',n})_{k'}$ such that
\[\left\{\begin{array}{ccl}
          \lambda_{k',n} &=& \mu_{k',n} \text{ if } k' \notin \{k, k^\star_m\}\;, \\
          \lambda_{k,n} &= & \mu_{k,n} + \delta_{n} \;,\\
          \lambda_{k^\star_m,n} & = & \mu_{k^\star_m,n} - \delta'_{n}\;,
         \end{array}
 \right.\]
 that satisfies 
	\begin{equation}\sum_{n \in [M]}w_{n,m}\left(\delta_{n} + \delta'_{n}\right) \geq \Delta'_{k,m}\;.\label{constraint}\end{equation}
Now arm $k$ is optimal for agent $m$. From Equation~\eqref{CD},
\[\sum_{n} \tau_{k,n} \frac{\delta_{n}^2}{2} + \sum_{n} \tau_{k^\star_m,n} \frac{(\delta'_{n})^2}{2} \geq \log\left(\frac{1}{2.4\delta}\right)\;.\]
Hence, it holds that 
	\[\inf_{\delta,\delta' : \eqref{constraint} \text{ holds}} \left[\sum_{n} \tau_{k,n} \frac{\delta_{n}^2}{2} + \sum_{n} \tau_{k^\star_m,n} \frac{(\delta'_{n})^2}{2}\right]\;. \] 
The infimum can be computed in closed form using constraint optimization. Introducing a Lagrange multiplier $\lambda$, from the KKT conditions, we get that, for any agent $n\;,$
\begin{eqnarray*}
\tau_{k,n}\delta_{n} - \lambda {w}_{n,m} & = & 0 \;,\\
\tau_{k^\star_m,n}\delta'_{n} - \lambda {w}_{n,m} & = & 0 \;,\\
\lambda\left(\sum_{n'}w_{n',m}\left(\delta_{n'} + \delta'_{n'}\right) - \Delta'_{k,m}\right) & = & 0 \;.
\end{eqnarray*}
Using furthermore that $\tau_{k,n}$ and $\tau_{k^\star_m,n}$ are positive,
	\[\delta_n = \frac{\Delta_{k,m}' w_{n,m}/\tau_{k,n}}{\sum_{n' \in [M]} w_{n',m}^2 \left(\frac{1}{\tau_{k,n'}} + \frac{1}{\tau_{k^\star_m,n'}}\right)} \text{ and } \ \ \delta'_n = \frac{\Delta_{k,m}' w_{n,m}/\tau_{k^\star_m,n}}{\sum_{n' \in [M]} w_{n',m}^2 \left(\frac{1}{\tau_{k,n'}} + \frac{1}{\tau_{k^\star_m,n'}}\right)}\;.\]

The conclusion follows by plugging these expressions to get the expression of the infimum.

\subsection{Regret Lower Bound: Proof of Theorem~\ref{thm:LBRegret}}\label{subapp:regret_lower_bound}

Given a bandit instance $\mu$, we can consider two sets of possible changes of distributions: changes that are allowed to change the distributions of optimal arms $(k^*_m)_{m \in [M]}$, and those that cannot
\begin{eqnarray*}
 \Alt(\mu) & := &\left\{\lambda= (\lambda_{k,m})_{k,m} \in \Rr^{K \times M}: \exists m \in [M], k \neq k^*_m : \sum_{n} \lambda_{k,m}w_{n,m} > \sum_{n}\lambda_{k^*_m,n}w_{n,m}\right\}\;,\\
 B(\mu) & := & \Alt(\mu) \bigcap\left\{\lambda= (\lambda_{k,m})_{k,m} \in \Rr^{K \times M}:\forall m \in [M], \lambda_{k^*_m,m} = \mu_{k^*_m,m}\right\}\;.
\end{eqnarray*}

For cumulative regret, the change-of-distribution lemma becomes an asymptotic result, stated below

\begin{lemma}\label{lem:CDasympt} Fix $\mu \in \Rr^{K \times M}$, and let us consider 
$\lambda \in \Alt(\mu)$. Then, for any $\varepsilon>0$, there exists $T_0 = T_0(\mu,\lambda,\varepsilon)\;,$ such that, for any $T \geq T_0$ 
\[\sum_{m,k} \bE_{\mu}[N_{k,m}(T)] \frac{(\mu_{k,m}-\lambda_{k,m})^2}{2} \geq (1-\varepsilon) \log(T)\;.\] 
\end{lemma}

\begin{proof} The proof uses a change-of-distribution, following a technique proposed by~\cite{GMS18}. Using the data processing inequality, and letting $\Hc_T$ be the observations available to the central server (which sees all local rewards under our assumptions), we have that 
\[\mathrm{KL}\left(\bP_{\mu}^{\Hc_T},\bP_{\lambda}^{\Hc_T}\right) \geq \kl\left(\bP_{\mu}(\Ec_T), \bP_{\lambda}(\Ec_T)\right),\]
where $\mathrm{KL}$ is the Kullback-Leibler divergence and $\bP_{\mu}^{\cH_T}$ is the distribution of the observation under the bandit model $\mu\;,$ and $\Ec_T$ is any event. Using that  
\[\mathrm{KL}\left(\bP_{\mu}^{\Hc_T},\bP_{\lambda}^{\Hc_T}\right)=\sum_{k,m} \bE_{\mu}[N_{k,m}(T)] \frac{\left(\mu_{k,m}- \lambda_{k,m}\right)^2}{2}\]
together with the lower bound $\kl(p,q)\geq (1-p)\log(1/(1-q)) - \log(2)$, where $\kl$ is the binary relative entropy and for any distributions $p,q$, yields 
\[\sum_{k,m} \bE_{\mu}[N_{k,m}(T)] \frac{\left(\mu_{k,m}- \lambda_{k,m}\right)^2}{2} \geq \left(1 - \bP_{\mu}(\Ec_T)\right)\log\left(\frac{1}{\bP_{\lambda}(\overline{\Ec_T})}\right) - \log(2)\;.\]
We now pick the event 
\[\mathcal{E}_T = \left\{N_{k_m^\star,T}(T) \leq \frac{T}{2}\right\}\;,\]
and use that $\mathcal{E}_T $ is very unlikely under $\mu$ as for any $\gamma \in (0,1)\;,$ 
\[\bP_{\mu}(\mathcal{E}_T) = \bP_{\mu} \left(\sum_{k \neq k_m^\star} N_{k,m}(T) \geq \frac{T}{2}\right) = \frac{2\sum_{k\neq k_m^\star}\bE_{\mu}[N_{k,m}(T)]}{T} = \frac{o_{T\rightarrow \infty}(T^{\gamma})}{T}\;,\]
and very likely under any $\lambda$ for which $k_m^\star$ is suboptimal as, for any $\gamma \in (0,1)\;,$  
\[\bP_{\lambda}(\overline{\mathcal{E}_T}) = \bP_{\lambda} \left(N_{k_m^\star,T}(T) > \frac{T}{2}\right) = \frac{2\bE_{\lambda}[N_{k_m^\star,m}(T)]}{T} = \frac{o_{T\rightarrow \infty}(T^{\gamma})}{T}\;,\]
where we exploit the fact that the algorithm is uniformly efficient (its regret and therefore its number of sub-optimal draws is $o(T^{\gamma})$ under any bandit model). The conclusion follows from some elementary real analysis to prove that the right hand side of the inequality is larger than $(1-\varepsilon)\log(T)$ for $T$ large enough (how larger depends in a complex way of $\mu,\lambda,\varepsilon$ and the algorithm). 
\end{proof}

At this point, we would really like to select the alternative model $\lambda$ that leads to the tightest inequality in Lemma~\ref{lem:CDasympt}. For example, using Lemma~\ref{lem:KKT} below, we find that the best way to make an arm $k \neq k_{m}^\star$ better than $k_{m}^\star$ consists in choosing 
\[\lambda_{k,n}(T) = \mu_{k,n} + \frac{\Delta_{k,m}w_{n,m} / \bE_{\mu}[N_{k,m}(T)]}{\sum_{m' \in [M]} \frac{w_{n',m}^2}{\bE_{\mu}[N_{k,n'}(T)]}}\;,\] for any $n \in [M]$
and $\lambda_{k',n}(T) = \mu_{k',n}$ for any arm $k' \neq k$. However, this choice of alternative $\lambda$ depends on $T$, hence we cannot apply Lemma~\ref{lem:CDasympt}, which is asymptotic in $T$ and holds for a fixed $\lambda\;.$
We have to be careful  to be able to exchange the $\lim\inf$ over $T$ and the infimum over alternatives in the constraints, and we will be able to do so only for changes of measures that are restricted to $B(\mu)\;.$ 

We first assume that $\liminf_{T\rightarrow\infty} \frac{\mathcal{R}(T)}{\log(T)}$ is finite, and call its value $\ell(\mu)$ (this is fine as otherwise any lower bound trivially holds). By definition of the $\lim\inf\;,$ there exists a sequence $(T_i)_{i \in \Nn}$ such that 
\[\liminf_{T\rightarrow \infty}\frac{\mathcal{R}(T)}{\log(T)} = \lim_{i \rightarrow \infty}\sum_{m \in [M],k\neq k_m^\star} \Delta_{k,m}'\frac{\bE_{\mu}[N_{k,m}(T_i)]}{\log(T_i)} = \ell(\mu)\;.\]
The fact that this sequence has a limit, and that the gaps are positive, implies that each sequence $\left({\bE_{\mu}[N_{k,m}(T_i)]}/{\log(T_i)}\right)_{i \in \Nn}$ is bounded. Therefore, there must exist a subsequence, that we denote $(T'_i)_{i \in \Nn}$ of $(T_i)_{i \in \Nn}$ such that 
\[\forall m \in [M], k\neq k_m^\star, \ \ \lim_{i\rightarrow \infty} \frac{\bE_{\mu}[N_{k,m}(T'_i)]}{\log(T'_i)} = c_{k,m}\;,\]
for some real values $(c_{k,m})_{k \in [K], m \in [M]}\;,$ and, in particular, 
\[\ell(\mu) = \sum_{m \in [M],k\neq k_m^\star}  \Delta_{k,m}' c_{k,m}\;.\]
Now, it follows from Lemma~\ref{lem:CDasympt} that, for any $\lambda \in \Alt(\mu)\;,$  
\[\liminf_{T\rightarrow \infty}\sum_{k,m}\frac{\bE_{\mu}[N_{k,m}(T)]}{\log(T)}\frac{(\mu_{k,m}-\lambda_{k,m})^2}{2} \geq 1\;.\]
But we have no idea about the behavior of the sequence $\left({\bE_{\mu}[N_{k,m}(T)]}/{\log(T)}\right)_{T \in \Nn}$ for $k = k_{m^\star}\;,$ this is why we have to consider only $\lambda \in B(\mu)\;,$ for which we deduce that  
\begin{eqnarray*}
\lim_{i\rightarrow \infty}\sum_{m \in [M],k\neq k_m^\star}\frac{\bE_{\mu}[N_{k,m}(T'_i)]}{\log(T'_i)}\frac{(\mu_{k,m}-\lambda_{k,m})^2}{2} & \geq& 1\;, \\
\sum_{m\in [M],k\neq k_m^\star}c_{k,m}\frac{(\mu_{k,m}-\lambda_{k,m})^2}{2} &\geq &  1\;.
\end{eqnarray*}
(the $\lim\inf$ being the lowest value of the limit of any subsequence). This proves that 
\[\ell(\mu) \geq \min_{c} \sum_{m \in [M],k\neq k_m^\star}  \Delta_{k,m}' c_{k,m}\;,\]
under the constraints that $c=(c_{k,m})_{k,m: k \neq k_m^\star}$ belongs to the constraint set 
\[\mathcal{C} = \left\{ (c_{k,m})_{k,m: k \neq k_m^\star} : \forall \lambda \in B(\mu),  \sum_{m \in [M],k\neq k_m^\star}c_{k,m}\frac{(\mu_{k,m}-\lambda_{k,m})^2}{2} \geq 1\right\}\;.\]
We now establish that 
\begin{equation}\label{eq:inclusion}\mathcal{C} \subseteq \bigcup_{k,m}\left\{ (c_{k,n})_{k,m: k \neq k_n^\star} : \sum_{n : k\neq k_n^\star} \frac{w_{n,m}^2}{c_{k,n}} \leq \frac{(\Delta_{k,m}')^2}{2}\right\}\;,\end{equation}
by selecting some well chosen elements in $B(\mu)$. 

First, for every $(m,k) \in [M] \times [K]$ such that $k\neq k_m^\star\;,$ for every $\delta = (\delta_n)_{n : k_n^\star \neq k}\;,$ we define an instance $\lambda^{\delta}$ by $\lambda_{k,n}^{\delta} = \mu_{k,n} + \delta_n$ for any $n \in [M]$ such that $k \neq k_n^\star\;,$ and $\lambda_{k,n}^{\delta} = \mu_{k,n}$ otherwise. We observe that $\lambda^{\delta}$ belongs to $B(\mu)$ if
\[\sum_{n \in [M] : k \neq k_n^\star} w_{n,m} \delta_n >  \Delta_{k,m}' \;,\]
as this leads to $\sum_{n \in [M]} w_{n,m} \lambda_{k,n} > \sum_{n \in [M]} w_{n,m} \lambda_{k_m^\star,n}\;,$ and arm $k_m^\star$ being sub-optimal in $\lambda^{\delta}\;.$  
For all $c\in \mathcal{C}\;,$ 
\[\min_{\delta : \sum_{n : k \neq k_n^\star} w_{n,m} \delta_n \geq  \Delta_{k,m}'} \sum_{n \in[M]: k_n^\star \neq k}c_{k,n}\frac{\delta_n^2}{2} \geq 1\;.\]
From Lemma~\ref{lem:KKT}, this leads to 
\[\sum_{n \in [M]: k_n^\star \neq k} \frac{w_{n,m}^2}{c_{k,n}} \leq \frac{(\Delta_{k,m}')^2}{2}\;.\]

Now we consider $(m,k) \in [M] \times [K]\;,$ such that $k = k_m^\star\;.$ In this case, we define an instance $\lambda^{\delta}$ by $\lambda_{k_m^\star,n}^{\delta} = \mu_{k_m^\star,n} - \delta_n$ for any $n \in [M]$ such that $k_m^\star \neq k_n^\star$, and $\lambda_{k,n}^{\delta} = \mu_{k,n}$ otherwise. We observe that $\lambda^{\delta}$ belongs to $B(\mu)$ if
\[\sum_{n \in [M] : k \neq k_n^\star} w_{n,m} \delta_n >  \min_{k' \in [K] : k' \neq k_m^\star}\Delta'_{k',m} = \Delta_{k_m^\star,m}'\;,\]
as this leads to $\sum_{n \in [M]} w_{n,m} \lambda_{k_{m}^\star,n} < \max_{k' \neq k_m^\star}\sum_{n} w_{n,m} \mu_{k',n}$ and arm $k_m^\star$ being sub-optimal in $\lambda^{\delta}$.  
For any $c\in \mathcal{C}$, 
\[\min_{\delta : \sum_{n : k \neq k_n^\star} w_{n,m} \delta_n \geq  \Delta_{k_m^\star,m}'} \sum_{n : k_n^\star \neq k_m^\star}c_{k,n}\frac{\delta_n^2}{2} \geq 1\;,\]
and Lemma~\ref{lem:KKT} leads to 
\[\sum_{n \in [M] : k_n^\star \neq k_{m}^\star} \frac{w_{n,m}^2}{c_{k,n}} \leq \frac{(\Delta_{k_m^\star,m}')^2}{2}\;.\]

In conclusion, for $c \in \mathcal{C}\;,$ we proved that, for every $(k,m) \in [K]\times [M]\;,$ 
\[\sum_{n \in [M]: k_n^\star \neq k} \frac{w_{n,m}^2}{c_{k,n}} \leq \frac{(\Delta_{k,m}')^2}{2}\;,\]
which proves the inclusion \eqref{eq:inclusion} and concludes the proof, as the minimum over a larger set is (potentially) smaller. 

\begin{lemma}\label{lem:KKT} Let $\mathcal{N}$ be a set of indices, and $(c_n)_{n\in \mathcal{N}}$ be all positive. The minimizer over $\delta \in \Rr^{|\mathcal{N}|}$ of 
$\sum_{n \in \mathcal{N}} c_n \frac{\delta_n^2}{2}\;,$
under the constraint that 
$\sum_{n \in \mathcal{N}} \delta_n w_{n,m} \geq d_{m}\;,$
satisfies
\[\forall n \in \mathcal{N}, \ \delta_n = \frac{d_m w_{n,m}/c_n}{\sum_{n'  \in \mathcal{N}} \frac{w_{n',m}^2}{c_{n'}}}\;,\]
and the minimum is equal to 
\[\frac{d_m^2}{2}\left(\sum_{n \in \mathcal{N}} \frac{w_{n,m}^2}{c_n}\right)^{-1}\;.\]
\end{lemma}

Lemma~\ref{lem:KKT} is proven using classical techniques for solving constrained minimization problems.

\section{Proof of Theorem~\ref{th:sample_complexity_bai}}\label{subapp:upper_bound_bai}\label{app:proofUB}

We recall the good event $\cE$ defined in Lemma~\ref{lem:validCI}. First, the following lemma ensures the correctness of Algorithm~\ref{algo:fpe_bai} on $\cE$, which holds with probability $1-\delta$ from Lemma~\ref{lem:validCI}.

\begin{lemma}\label{lem:delta_correctness} On event $\cE\,$, when stopping, \textcolor{black}{W-CPE}-BAI 
outputs $\hat{k}_m = k^\star_m$ for each agent $m$. 
\end{lemma}

\begin{proof}
	On event $\Ec$, it is not possible that one agent $m$ eliminates arm $k_m^\star$ from its set $B_m(r+1)$ at any round $r$~; otherwise, if $j_m(r) \in \arg\max_{j \in B_m(r)} \{ \hat{\mu}'_{j,m}(r) - \Omega_{j,m}(r) \}\;,$ $j_m(r) \neq k^\star_m\;,$ then the elimination criterion and event $\Ec$ imply that 
	$\mu'_{k^\star_m,m} < \mu'_{j_m(r),m}\;,$ which is absurd. 
\end{proof}

Then we upper bound the exploration cost when $\Ec$ holds. We denote by $R$ the (random) number of rounds used by the algorithm, and, for all $m \in [M]$ and $k\neq k_m^\star\;,$ by $R_{k,m}$ the (random) last round in which $k$ is still a candidate arm for player $m$ 
\[R_{k,m} := \sup \left\{ r \geq 0 : k \in B_m(r)\right\}\;.\]
By definition of Algorithm~\ref{algo:fpe_bai}, $R = \max_{m}\max_{k \neq k_m^\star} R_{k,m}$. 
We first provide upper bounds on $R_{k,m}$ and $R$. To achieve this, we introduce the following notation for any $k \in [K],m \in [M]\;,$
\begin{eqnarray*}
	r_{k,m} &:=& \min \left\{ r \geq 0 : 4 \times 2^{-r} < \Delta'_{k,m} \right\} \text{ and }r_{\max} := \max_{m \in [M]}\max_{k \neq k^\star_m} r_{k,m}\;.
\end{eqnarray*}
The following upper bounds can be easily checked. 
\begin{lemma}\label{lem:UBr}
	$\forall k,m, r_{k,m} \leq \log_2\left({8}/{\Delta_{k,m}'}\right)\;$ and $r_{\max}\leq \log_2\left({8}/{\Delta_{\min}'}\right)\;.$
\end{lemma}

Using the fact that \textcolor{black}{W-CPE}
-BAI is only halving the proxy gaps of arms that are not eliminated, we can write down the value of the proxy gaps for these arms.

\begin{lemma}\label{lem:deltat_value}
	$\forall m \in [M]$ and $k \in B_m(r)$, $\Deltat_{k,m}(r)=2^{-r}\;.$
\end{lemma}

Using the important relationship between proxy gaps and the confidence width as established in Lemma~\ref{lem:upperbound_omega}, we can further show that 

\begin{lemma}\label{lem:round_elimination}
	On $\Ec$, for any $m \in [M],k\neq k^\star_m$, $R_{k,m} \leq r_{k,m}\;.$
\end{lemma}

\begin{proof}
Assume $\Ec$ holds. For any suboptimal arm $k$ for agent $m$, at round $r = r_{k,m}$, 
	if $k \not\in B_m(r)$, then $R_{k,m} < r_{k,m}\;.$ Otherwise, if $k \in B_m(r)$, then we know that $k_m^\star \in B_m(r)$, as event $\cE$ holds (see the proof of Lemma~\ref{lem:delta_correctness}). Then 
\begin{equation*}
    \begin{split}
	    \hat{\mu}'_{k,m}(r) +\Omega_{k,m}(r) & \leq_{(1)} \mu'_{k,m}+2\Omega_{k,m}(r)\\
	    & \leq_{(2)} \mu'_{k,m}+2\Deltat_{k,m}(r) = \mu'_{k,m}+4\Deltat_{k,m}(r)-2\Deltat_{k,m}(r)\\
	    & <_{(3)} \mu'_{k,m}+\Delta'_{k,m}-2\Deltat_{k,m}(r) = \mu'_{k^\star_m,m}-2\Deltat_{k,m}(r)\\
	    & \leq_{(1)} \hat{\mu}'_{k^\star_m,m}(r)+\Omega_{k^\star_m,m}(r)-2\Deltat_{k,m}(r)\\
	    & \leq_{(2),(4)} \max_{j \in B_m(r)} \left\{ \hat{\mu}'_{j,m}(r)-\Omega_{j,m}(r) \right\} + 2 \cdot 2^{-r} - 2\cdot 2^{-r}\;,\\
	    \implies \hat{\mu}'_{k,m}(r) +\Omega_{k,m}(r) & < \max_{j \in B_m(r)} \left\{ \hat{\mu}'_{j,m}(r)-\Omega_{j,m}(r) \right\}\;,\\
    \end{split}
\end{equation*}

	where (1) is using that event $\Ec$ holds~; (2) is using Lemma~\ref{lem:upperbound_omega}~; (3) is using the definition of $r=r_{k,m}$, the fact that $k\in B_m(r)$ and Lemma~\ref{lem:deltat_value} and (4) is using that $k,k^\star_m \in B_m(r)$ and Lemma~\ref{lem:deltat_value}. It follows that  $k \not\in B_m(r_{k,m}+1)$ 
	and $R_{k,m} \leq r_{k,m}$.\\
\end{proof}

The previous lemma straightforwardly implies that

\begin{corollary}\label{lem:UBR}
	$R \leq r_{\max} \leq \log_2(8/\Delta_{\min}')\;.$
\end{corollary}

Moreover, it also permits to prove that, in the last round $R$, the proxy gaps are lower bounded by the gaps.
%

\begin{corollary}\label{cor:lubounds_gap}
	At final round $R$, and for any agent $m$ and suboptimal (with respect to $m$) arm $k \neq k^\star_m$, if $\Ec$ holds,
	\[\Deltat_{k,m}(R) \geq \frac{1}{8}\Delta'_{k,m}\;.\]
\end{corollary}

\begin{proof}
	If $R < r_{k,m}\;,$ by definition of $r_{k,m}\;,$ we have that $\Deltat_{k,m}(R) \geq (1/4) \Delta_{k,m}' \geq (1/8) \Delta_{k,m}'\;.$ If $R \geq r_{k,m}\;,$ we first observe that $\Deltat_{k,m}(R) = \Deltat_{k,m}(R_{k,m}) = (1/2)\Deltat_{k,m}(R_{k,m} - 1)$ by definition of the algorithm (the gaps remain frozen when an arm is eliminated, and they are halved otherwise). As $R_{k,m} - 1 < r_{k,m}$ by Lemma~\ref{lem:round_elimination}, by definition of $r_{k,m}\;,$ it follows that  \[4\Deltat_{k,m}(R_{k,m}-1) > \Delta_{k,m}'\;\]
and we conclude that $\Deltat_{k,m}(R) \geq (1/8) \Delta_{k,m}'\;.$  
\end{proof}

Note that, for any $m \in [M]\;,$ $\Deltat_{k^\star_m,m}(R) = \min_{k \neq k^\star_m} \Deltat_{k,m}(R)$ by Algorithm~\ref{algo:fpe_bai}. Now, for any $m \in [M],k \in[K]\;,$ using Corollary~\ref{cor:lubounds_gap}, and the fact that the proxy gaps are non-increasing between two consecutive phases, we get
\[\forall k \in [K], \forall m\in [M], \forall r \leq R, \ \Deltat_{k,m}(r) \geq \Deltat_{k,m}(R)  \geq \frac{\Delta_{k,m}'}{8}\;.\]




Using Lemma~\ref{lem:12} and the fact that proxy gaps are non-increasing, for any round $r \leq R\;,$ the optimal allocation $t(r) \in \oracle\left(\left(\sqrt{2}\Deltat_{k,m}(r)\right)_{k,m}\right)$ satisfies 
\[\sum_{k,m} t_{k,m}(r) \leq 32\sum_{k,m} t'_{k,m}\;,\]
where $t' \in \oracle\left(\Delta'\right)\;,$ hence
\begin{equation}\max_{r \leq R} \left[\sum_{k,m} t_{k,m}(r)\right] \leq 32 \widetilde T^\star_W(\mu)\;.\label{eq:upper_bound_t}\end{equation}

For every $k \in [K],m \in [M]$, we now introduce 
\[r'_{k,m} := \sup \{ r \leq R : d_{k,m}(r) \neq 0\}\;,\]
so that $n_{k,m}(R) = n_{k,m}(r'_{k,m})\;.$ Using Lemma~\ref{lem:UBn} stated below, and the fact that function $\beta_{\delta}$ is nondecreasing in each coefficient of its argument (see its definition in Lemma~\ref{lem:validCI}), 
\begin{eqnarray*}n_{k,m}(R) = n_{k,m}(r'_{k,m}) &\leq& t_{k,m}(r'_{k,m}) \beta_{\delta}(n_{k,\cdot}(r'_{k,m})) + 1 \\
 & \leq & t_{k,m}(r'_{k,m}) \beta_{\delta}(n_{k,\cdot}(R)) + 1\;.
\end{eqnarray*}


\begin{lemma}\label{lem:UBn}
For any $k,m,r \geq 0$, either $d_{k,m}(r)=0$, or 
$n_{k,m}(r) = n_{k,m}(r-1)+d_{k,m}(r) < t_{k,m}(r) \beta_{\delta}(n_{k,\cdot}(r))+1\;.$
\end{lemma}

\begin{proof}
	At fixed $r \geq 0$, for any set $S \subseteq [K] \times [M]$, let us prove by induction on $|S|\geq1$~\footnote{For any $x \in \Nn^{M}\;,$ $(x)_+ := (\max(0, x_m))_{m \in [M]}\;,$ and $\mathds{1}_S$ is the indicator function of set $S\;.$}
	\begin{eqnarray*}
	   \forall k,m, & & d'_{k,m}(r) := \left(d_{k,m}(r)-\mathds{1}_{S}((k,m))\right)_+ \\
	   & \implies& \forall (k,m) \in S, \frac{n_{k,m}(r-1) + d'_{k,m}(r)}{\beta_{\delta}(n_{k,\cdot}(r-1)+d_{k,\cdot}(r))} < t_{k,m}(r)\\
	   & & \text{ or }  d_{k,m}(r)=0\;. 
	\end{eqnarray*}

	\paragraph{At $|S|=1\;:$} Let us denote $S = \{(k',m')\}\;.$ If $d_{k',m'}(r)=0\;,$ then it is trivial. Otherwise, $\sum_{k,m} d'_{k,m}(r) < \sum_{k,m} d_{k,m}(r)\;,$ and then, by minimality of solution $d(r)\;,$ at least one constraint from the optimization problem of value $\sum_{k,m} d_{k,m}(r)$ has to be violated. For any $(k,m) \not\in S\;,$ by definition of $d(r)$ and nondecreasingness of $\beta_{\delta}\;,$
	\begin{eqnarray*}
		n_{k,m}(r-1) + d'_{k,m}(r) &= &n_{k,m}(r-1)+d_{k,m}(r)\\
		& \geq & t_{k,m}(r)\beta_{\delta}(n_{k,\cdot}(r-1)+d_{k,\cdot}(r))\\
		& \geq & t_{k,m}(r)\beta_{\delta}(n_{k,\cdot}(r-1)+d'_{k,\cdot}(r))\;. 
	\end{eqnarray*}
	That means, necessarily the only constraint that is violated is the one on $(k',m')\;.$ Using the nondecreasingness of $\beta_{\delta}\;:$
	\begin{eqnarray*}
		n_{k',m'}(r-1) + d_{k',m'}(r)-1 &=& n_{k',m'}(r-1)+d'_{k',m'}(r)\\
		& <&  t_{k',m'}(r)\beta_{\delta}(n_{k',\cdot}(r-1)+d'_{k,\cdot}(r))\\
		& \leq & t_{k',m'}(r) \beta_{\delta}(n_{k',\cdot}(r-1)+d_{k,\cdot}(r))\;.
	\end{eqnarray*}
	Combining the two ends of the inequality proves the claim.\\

\paragraph{At $|S|>1\;:$} At fixed $(k',m') \in S\;,$ we can apply the claim to $S \setminus \{(k',m')\}$. Moreover, if $d_{k',m'}(r)=0\;,$ then the claim is proven. Otherwise, by appealing to the extremes,
	\[ n_{k',m'}(r-1)+d'_{k',m'}(r) \geq t_{k',m'}(r) \beta_{\delta}(n_{k',\cdot}(r-1)+d_{k,\cdot}(r))\;.\]
Let us then consider the following allocation:
		\[ \forall k,m, d''_{k,m}(r) := d_{k,m}(r)-\mathds{1}_{\{(k',m')\}}((k,m))\;. \]
	It can be checked straightforwardly -- using the nondecreasingness of $\beta_{\delta}$ -- that $d''$ satisfies all required constraints for any pair $(k,m) \in [K] \times [M]\;,$ and that $\sum_{k,m} d''_{k,m}(r) = \sum_{k,m} d_{k,m}(r) -1\;,$ which, by minimality of $d\;,$ is absurd. Then the claim is proven for $|S|>1\;.$ Then Lemma~\ref{lem:UBn} is proven by considering $S = [K] \times [M]\;.$

\end{proof}

By summing the upper bound on $(n_{k,m}(R))_{k,m}$ over $[K] \times [M]$, we can upper bound the exploration cost $\tau$ as


\begin{eqnarray*}
	\tau := \sum_{k,m} n_{k,m}(R) & \le & \sum_{k,m} t_{k,m}(r'_{k,m}) \beta_{\delta}(n_{k,\cdot}(R))+KM\\
	& \leq & \sum_{k,m} t_{k,m}(r'_{k,m}) \beta^*(\tau) + KM\\
	& \leq &  \sum_{k,m} \sum_{r \leq R}t_{k,m}(r)  \beta^*(\tau) + KM\\
	& \leq & R \max_{r \leq R}\left[\sum_{k,m} t_{k,m}(r)\right] \beta^*(\tau) + KM \\
	& \leq &  \log_{2}\left(8/\Delta_{\min}'\right)\max_{r \leq R}\left[\sum_{k,m} t_{k,m}(r)\right] \beta^*(\tau) + KM
\end{eqnarray*}

where we use Corollary~\ref{lem:UBR} and introduce the quantity
\[\beta^*(\tau) := \beta_{\delta}\left(\tau \mathds{1}_{M}\right) = 2\left(g_M\left(\frac{\delta}{KM}\right) + 2 M \ln \left(4 + \ln \left(\tau\right)\right)\right) \text{ where } \forall n \in [M], \mathds{1}_{M}(n) = 1\;.\]
Using Equation~\eqref{eq:upper_bound_t} and Lemma~\ref{lem:ttilde_optimal}, 
\[\tau \leq 32\widetilde T^\star_W(\mu) \log_{2}\left(8/\Delta_{\min}'\right)\beta^*(\tau) + KM \leq 32T^\star_W(\mu)\log_2(8/\Delta'_{\min})\beta^*(\mu)+KM\;.\]
Therefore, $\tau$ is upper bounded by 
\[\sup \left\{ n \in \Nn^\star : n \leq 32 T^\star_W(\mu)\log_2(8/\Delta'_{\min}) \beta^*(n) + KM\right\}\;.\]
Applying Lemma $15$ in~\cite{kaufmann2021adaptive} with 
\begin{eqnarray*}
	\Delta &=& \left(\sqrt{32 T^\star_W(\mu)\log_{2}\left(8/\Delta_{\min}'\right)}\right)^{-1} \;,\\
a & = & KM + 2g_M\left(\frac{\delta}{KM}\right) \;,\\
b & = & 4M \;,\\
c & = & 4 \;,\\
d & = & e^{-1} \text{ ( using }\forall n, \log(n) \leq n e^{-1} \text{ )}\;,
\end{eqnarray*}
$\tau$ is upper bounded by 
	\begin{eqnarray*} \hat{T}_W(\mu) &:=& 32 T^\star_W(\mu)\log_{2}\left(8/\Delta_{\min}'\right) \left[ KM+2g_M\left(\frac{\delta}{KM}\right) \right.
	\\ 
	&&\left.+4M\ln \left( 4 + 1,024\frac{(T^\star_W(\mu)\log_{2}\left(8/\Delta_{\min}'\right))^2}{e} \left( KM+2g_M\left(\frac{\delta}{KM}\right) + 4M(2+\sqrt{e})\right)^2 \right) \right]\;,\end{eqnarray*}
which satisfies $\tau \hat{T}_W(\mu) \leq a+b\ln(c+d\tau \hat{T}_W(\mu) )\;.$ Using that $g_{M}(x) \simeq x + M\log\log(x)$ in the regime of small values of $\delta\;,$ we obtain that 
\[ \hat{T}_W(\mu) = 32T^\star_W(\mu)\log_{2}\left(8/\Delta_{\min}'\right)\log(1/\delta) + o_{\delta \rightarrow 0}\left(\log(1/\delta)\right).\]
The upper bound on the communication cost follows from the upper bound on the number of phases given in Corollary~\ref{lem:UBR}.

\section{Supplementary Lemmas and Proofs}~\label{app:proofs}

We report here technical lemmas and their proofs, ordered by section.

\subsection{\textcolor{black}{Collaborative} 
Best-Arm Identification}\label{subapp:proofs_federated_bai}

\begin{lemma} Introducing the quantity
\[\widetilde{T}^\star_W(\mu) := \min_{t \in (\Rr^+)^{K \times M}} \sum_{(k,m) \in [K] \times [M]} t_{k,m} \text{ s.t. } \forall m \in [M], k \in [K], \sum_{n \in [M]}\frac{w_{n,m}^2}{t_{k,n}} \leq \frac{\left(\Delta'_{k,m}\right)^2}{2} \;,\]
it holds that $\widetilde{T}^\star_W(\mu) \leq T^\star_W(\mu) \leq 2 \widetilde{T}^\star_W(\mu)\;.$ 
\end{lemma}

\begin{proof}
	Let us denote by $\mathcal{C}$ and $\widetilde{\mathcal{C}}$ the two constraint sets such that $T^\star_W(\mu) = \min \left\{ \sum_{k,m} t_{k,m} \mid t \in \mathcal{C} \right\}$ and $\widetilde{T}^\star_W(\mu) = \min \left\{ \sum_{k,m} t_{k,m} \mid t \in \widetilde{\mathcal{C}} \right\}\;.$ The inequality $\widetilde{T}^\star_W(\mu) \leq T^\star_W(\mu)$ is obtained by noticing that $\mathcal{C} \subseteq \widetilde{\mathcal{C}}$. To prove the other inequality, we consider $\widetilde{\tau} \in \arg\min \left\{ \sum_{k,m} t_{k,m} \mid t \in \widetilde{\mathcal{C}} \right\}\;.$   
	Then, for any agent $m \in [M]\;,$ arm $k \neq k^\star_m\;,$
	\begin{equation*}
		\begin{split}
			\sum_{n \in [M]} w_{n,m}^2 \left( \frac{1}{2\widetilde{\tau}_{k,n}}+\frac{1}{2\widetilde{\tau}_{k^\star_m,n}} \right) & = \frac{1}{2}\underbrace{\left( \sum_{n \in [M]} \frac{w_{n,m}^2}{\widetilde{\tau}_{k,n}} \right)}_{\leq \left(\Delta'_{k,m}\right)^2/2}+\frac{1}{2}\underbrace{\left( \sum_{n \in [M]} \frac{w_{n,m}^2}{\widetilde{\tau}_{k^\star_m,n}} \right)}_{\leq \left(\Delta'_{k^\star_m,m}\right)^2/2 := \min \left\{ \left(\Delta'_{k',m}\right)^2/2 \mid k' \neq k^\star_m \right\}}\\
			& \leq \left(\Delta'_{k,m}\right)^2/2\;.
		\end{split}
	\end{equation*}
	Then $2\widetilde{\tau} \in \mathcal{C}\;,$ therefore by minimality, $ T^\star_W(\mu) \leq 2\widetilde{T}^\star_W(\mu)\;.$
\end{proof}

\begin{lemma}
	Consider $\Delta\;,$ $\Delta' \in (\Rr^+)^{K \times M}$, such that $\tau \in \oracle(\Delta)$ and $\tau' \in \oracle(\Delta')\;.$ Then\\
	\paragraph{(i).} If there exists $\alpha >0$ such that: $ \forall k \in [K], \forall m \in [M], \alpha\Delta_{k,m} \leq \Delta'_{k,m}$, 
	\[\sum_{k,m} \tau'_{k,m} \leq \frac{1}{\alpha^2} \sum_{k,m} \tau_{k,m} \;.\]

	\paragraph{(ii).} If there exists $\beta > 0$ such that: $\forall k \in [K], \forall m \in [M], \Delta'_{k,m} \leq \beta \Delta_{k,m}$. Then
	\[\frac{1}{\beta^2} \sum_{k,m} \tau_{k,m} \leq \sum_{k,m} \tau'_{k,m}\;.\]
\end{lemma}

\begin{proof}
	The proof follows from the fact that $\tau$ and $\tau'$ are \textit{minimal}. In particular, to prove \textbf{(ii)}, let $\tau''_{k,m} = \beta^2 \tau'_{k,m}$ for any $k \in [K]\;,$ $m \in [M]\;.$ Then, for any agent $m$ and arm $k$,

	\[\sum_{n \in [M]}\frac{w^2_{n,m}}{\tau''_{k,n}} = \sum_{n \in [M]}\frac{w^2_{n,m}}{\beta^2\tau'_{k,n}} \le \frac{1}{2}\left(\frac{\Delta'_{k,m}}{\beta}\right)^2 \le \frac{\Delta^2_{k,m}}{2}\;. \]

	By minimality of $\tau\;,$ 
	\[\sum_{m \in [M]}{\tau_{k,m}} \le \sum_{m \in [M]}{\tau''_{k,m}} = \beta^2 \sum_{n \in [M]}{\tau'_{k,m}}\;. \]
\textbf{(i)} similarly follows.
\end{proof}

\subsection{A Near-Optimal Algorithm For Best Arm Identification}\label{subapp:proofs_federated_bai_algo}

\begin{lemma}
Let us define 
	\[\beta_{\delta}(N) := 2\left( g_M\left(\frac{\delta}{KM}\right) + 2 \sum_{m=1}^{M} \ln(4+\ln(N_{m})) \right)\;,\]
	for any $N \in (\Nn^*)^{M}$, where 
	\[ \forall \delta \in (0,1), g_{M}(\delta) := M\mathcal{C}^{g_G}\left(\log(1/\delta)/M\right) \;,\]
	\[ \forall x > 0, \mathcal{C}^{g_G}(x) := \min_{\lambda \in (0.5,1)} \frac{g_G(\lambda)+x}{\lambda} \;,\]
	\[ \text{ and } \forall \lambda \in (0.5,1), g_G(\lambda) := 2\lambda-2\lambda \log(4\lambda)+\log(\zeta(2\lambda))-0.5\log(1-\lambda) \;,\] 
	where $\zeta$ is the Riemann zeta function. Then, the good event \[\mathcal{E} := \left\{\forall r \in \mathbb{N}, \forall m, \forall k, \left|\hat{\mu}_{k,m}'(r) - \mu_{k,m}'\right| \leq \Omega_{k,m}(r)\right\}\;.\]
holds with probability larger than $1-\delta$.
\end{lemma}

\begin{proof}
Using Proposition 24 from~\cite{kaufmann2018mixture} on $\mu'_{k,m}$, for any arm $k$ and agent $m$, directly yields
\[ \bP\left( \exists r \geq 0, |\hat{\mu}'_{k,m}(r)-\mu'_{k,m}| > \sqrt{2\left( g_M\left(\frac{\delta}{KM}\right) + 2 \sum_{n=1}^{M} \ln(4+\ln(n_{k,n}(r))) \right)\sum_{n} \frac{w_{n,m}^2}{n_{k,n}(r)}} \right) \leq \frac{\delta}{KM} \]
(using the notation of the paper, consider $\mu = \mu_{k,\cdot}$ and $c = W_{\cdot,m}$). Then all that is needed to conclude is to apply a union bound on $[K] \times [M]$

\begin{equation*}
\begin{split}
    \bP\left( \cE^c \right) & = \bP\left( \exists m \in [M], \exists k \in [K], \exists r \geq 0, |\hat{\mu}'_{k,m}(r)-\mu'_{k,m}| > \sqrt{2\beta_{\delta}(n_{k,\cdot}(r))\sum_{n} \frac{w_{n,m}^2}{n_{k,n}(r)}} \right)\\
    & \leq \sum_{m \in [M]} \sum_{k \in [K]} \frac{\delta}{KM} \leq \delta\;.
\end{split}
\end{equation*}


\end{proof}

\subsection{Regret Lower Bound}\label{subapp:proofs_regret}

\begin{lemma}
Introducing the quantity 
\[\widetilde{C}^\star_W(\mu) := \min_{c \in (\Rr^+)^{K\times M}} \sum_{k=1}^{K}\sum_{m=1}^M c_{k,m}\Delta_{k,m}' \text{ s.t. } \forall k \in [K], \forall m \in [M], \sum_{n =1}^{M} \frac{w_{n,m}^2}{c_{k,n}} \leq \frac{\Delta'_{k,m}}{2\sigma^2}\;,\]

it holds that $C^\star_W(\mu)\le\widetilde{C}^\star_W(\mu)\le 4C^\star_W(\mu)\;.$
\end{lemma}

 \begin{proof} 

Let $c$ and $\widetilde{c}$ be the solutions to the optimization problems of $C^\star_W(\mu)$ and $\widetilde{C}^\star_W(\mu)$, respectively. Note that, for any agent $m\;,$ $c_{m,k^*_m}=+\infty$ because, in the optimization problem related to the regret lower bound, these terms do not contribute to the objective. 
The lower bound follows from the definition of $c$ and the fact that $(\widetilde{c}_{k,m})_{m, k\neq k^*_m}$ satisfy the same constraints as $(\widetilde{c}_{k,m})_{m, k\neq k^*_m}$. 

Next, we prove the upper bound on $\widetilde{C}^\star_W(\mu)\;.$ For any $k\in[K]\;,$ define $\Sc^*_k=\{m: k=k^*_m\}$ and $\Sc_k=\{m: k\not= k^*_m\}$ (note that for any $k \in [K]$, $\{\Sc^*_k, \Sc_k\}$ is a partition of $[M]$).
For $k\in[K]$ and $m\in \Sc_k\;,$ let $c'_{k,m}=2c_{k,m}\;.$ For $k\in[K]$ and $m\in \Sc^*_k$, let $c'_{k,m} = c^1_{k,m}$, where 

\begin{eqnarray*}
c^1 &\in &\arg\min_{\tau \in (\Rr^{+})^{K \times M}} \sum_{k=1}^K\sum_{m\in\Sc^*_k} \tau_{k,m}\Delta'_{k,m} \text{ s.t.} \forall k\in[K],m\in[M], \sum_{n\in \Sc^*_k}\frac{w^2_{n,m}}{\tau_{k,n}}\le \frac{(\Delta'_{k,m})^2}{4}\;.
\end{eqnarray*}

Note that $c'_{k,m}$ satisfy the same constraints as $\widetilde{c}_{k,m}$:
$\forall k\in[K],m\in[M], \sum_{n=1}^M\frac{w^2_{n,m}}{c'_{k,n}}\le \frac{(\Delta'_{k,m})^2}{2}\;.$
We thus have

\begin{eqnarray*}
\sum_{k=1}^K\sum_{m=1}^M\widetilde{c}_{k,m}\Delta'_{k,m} 
&\le& \sum_{k=1}^K\sum_{m=1}^Mc'_{k,m}\Delta_{k,m} \\
&=&
 \sum_{k=1}^K\sum_{m\in\Sc_k}c'_{k,m}\Delta_{k,m} +
\sum_{k=1}^K\sum_{m\in\Sc^*_k}c'_{k,m}\Delta_{k,m}\\
&=&
2\sum_{k=1}^K\sum_{m\in\Sc_k}c_{k,m}\Delta'_{k,m} + \sum_{k=1}^K\sum_{m\in\Sc^*_k}c'_{k,m}\Delta_{k,m}\\
&=&
{ 2C^\star_W(\mu) + \sum_{k=1}^K\sum_{m\in\Sc^*_k}c^1_{k,m}\Delta'_{k,m}}\\
&\le&
 4C^\star_W(\mu)\;.
\end{eqnarray*}

The first inequality holds by minimality of $\widetilde{c}\;.$ To prove the second inequality, for all $m\in[M]\;,$ let us define $c''_{k^*_m,m}=2\sum_{k\neq k^*_m}c_{k,m}\;,$ and $c''_{k,m}=\infty$ for $k \neq k^\star_m\;.$ Note that $c''$ satisfy the constraints in the definition of $c^1\;.$ Thus, by minimality of $c^1\;,$ and using the fact that, for any $m \in [M]\;,$ $\Delta'_{k^\star_m,m} := \min_{k \neq k^\star_m} \Delta'_{k,m}\;,$ we have 

\begin{eqnarray*}
\sum_{k=1}^K\sum_{m\in\Sc^*_{k}}c^1_{k,m}\Delta'_{k,m} &\le&
\sum_{k=1}^K\sum_{m\in\Sc^*_{k}}c''_{k,m}\Delta'_{k,m} \\\nonumber 
&=& 2\sum_{k=1}^K\sum_{k,m\in\Sc^*_{k}}\left(\sum_{j\neq k^*_m}c_{j,m}\right)\Delta'_{k,m}\\\nonumber
&=&2\sum_{k=1}^K\sum_{m\in\Sc_k}c_{k,m}\Delta'_{k^*_m,m}\\\nonumber
&\le&2\sum_{k=1}^K\sum_{m\in\Sc_k}c_{k,m}\Delta'_{k,m}\\\nonumber
&=& 2C^\star_W(\mu)\;,
\end{eqnarray*}
which completes the proof.

\end{proof}

\section{Extension to Top-$N$ Identification}\label{app:topN}

A generalization of the best arm identification problem is Top-$N$ identification, which is the problem of finding the $N$ optimal arms (for mixed rewards) for each agent. For any model $\mu \in \Rr^{K \times M}\;,$ weight matrix $W \in [0,1]^{M \times M}\;,$ such that $\mu' = \mu W\;,$ any agent $m\;,$ and positive integer $N \leq K\;,$ let us define~\footnote{We define operation $\max^N$ such that $\max^N_{i \in S} f(i)$ is the $N^{th}$ (without multiplicity) greatest value in set $\{f(i) : i \in S\}$ for any function $f : S \mapsto \Rr$.}

\[ S^\star_m := \left\{ k \in [K] \mid \mu'_{k,m} \geq \max^{N}_{k'\in [K]} \mu'_{k',m} = \max^N_{k ' \in [K]} \sum_{n \in [M]} w_{n,m}\mu_{k',n} \right\} \;.\]

In this case, an algorithm for Top-$N$ identification returns at the end of the exploration phase a set of $N$ arms denoted $\hat{S}_m$ for agent $m\;.$ $\delta$-correctness is defined as follows

\[ \bP_{\mu}\left( \forall m \in [M], \hat{S}_m \subseteq S^\star_m \right) \geq 1-\delta \;.\]

Note that there might be more than $N$ arms in a given set $S^\star_m, m \in [M]\;.$ Similarly to best arm identification, we also assume here that the set of top-$N$ arms is unique -- that is, for any $m$, $|S^\star_m|=N\;.$ 

\paragraph{Lower Bound for Top-$N$ Identification} For Top-$N$ identification, one can prove, similarly to the proof of Theorem~\ref{thm:LBBAI} -- using Lemma $1$ in~\cite{reda2021dealing} to define the set of alternative models -- the following result, which is valid for Gaussian rewards with fixed variance $\sigma^2=1\;,$ and any weight matrix $W$ such that all diagonal coefficients are positive,

\begin{theorem}\label{thm:LBTopN} 
	Let $\mu$ be a fixed matrix of means in $\mathbb{R}^{K \times M}\;.$ For any $\delta \in (0,1/2]\;,$ let $\cA$ be a $\delta$-correct algorithm under which each agent communicates each reward to the central server after it is observed, and let us denote for any $k\in [K],m\in[M]\;,$ $\tau_{k,m} := \bE_{\mu}^{\cA}\left[ N_{k,m}(\tau) \right]\;,$ where $\tau$ is the stopping time. For any $m \in [M], k\not\in S_m^\star, l \in S_m^\star \;,$ denoting $\mu'=\mu W$, it holds that  
	\[\sum_{n}w_{n,m}^2\left(\frac{1}{\tau_{k,n}} + \frac{1}{\tau_{l,n}}\right) \leq \frac{\left(\mu'_{k,m}-\mu'_{l,m}\right)^2}{2\log(1/(2.4\delta))} \;,\]
	and therefore $\text{Exp}_{\mu}(\cA) \geq N^{\star}_W(\mu) \log\left(\frac{1}{2.4\delta}\right)\;,$ where 

	\begin{equation*}
		\begin{split}
			N^{\star}_W(\mu) := & \min_{t \in \Rr^{K \times M}} \sum_{(k,m) \in [K] \times [M]} t_{k,m} \\
			& \text{s.t. } \forall m, k\not\in S^\star_m(\mu), l \in S^\star_m(\mu), \ \sum_{n \in [M]}w_{n,m}^2\left(\frac{1}{t_{k,n}} + \frac{1}{t_{l,n}}\right) \leq \frac{\left(\mu'_{k,m}-\mu'_{l,m}\right)^2}{2} \;.\\ 
		\end{split}
	\end{equation*}
\end{theorem}

\begin{proof}
	As mentioned, let us use Lemma $1$ from~\cite{reda2021dealing} to define the set of alternative models to $\mu$ in Top-$N$ identification. If, for any agent $m$, $S^\star_m(\mu)$ is the set of its top-$N$ arms (of size $N$) with respect to mixed rewards, then
	\begin{eqnarray*}
	\mathrm{Alt}(\mu) & := & \left\{ \lambda \in \Rr^{K \times M} : \exists m, S^\star_m(\mu) \not\subseteq S^\star_m(\lambda) \right\}\\
	& = & \left\{\lambda \in \Rr^{K \times M} : \exists m, \exists k \not\in S^\star_m(\mu), \exists l \in S^\star_m(\mu)  : \lambda'_{k,m} > \lambda'_{l,m} \right\} \;,
	\end{eqnarray*}
where $\lambda'_{k,m} := \sum_{n \in [M]} w_{n,m}\lambda_{k,n}$ for any arm $k$ and agent $m\;.$ If we assume that stopping time $\tau$ is almost surely finite under $\mu$ for algorithm $\cA\;,$ then let event $\Ec_\mu := \left\{\exists m :  \hat{S}_m \not\subseteq S_m^\star(\mu)\right\}\;.$ Using the $\delta$-correctness of algorithm $\cA$, where $\delta \leq 1/2\;,$ by Theorem $1$ from~\cite{garivier2016optimal}, 
	\begin{equation}\frac{1}{2} \sum_{k,m} \tau_{k,m} (\mu_{k,m} - \lambda_{k,m})^2 
	\geq \log\left(\frac{1}{2.4\delta}\right)\;.\label{eq:cd_topn}\end{equation}

	Similarly to the best arm identification case, we can show that, since all diagonal coefficients of $W$ are positive, for any $k\in[K],m\in[M]\;,$ $\tau_{k,m}>0\;.$ Consider now a fixed agent $m$, and two arms $k \not\in S^\star_m(\mu)$, $l \in S^\star_m(\mu) \;.$ We will build an alternative model $\lambda$, similar enough to $\mu$, where only arms $k$ and $l$ are modified for all agents, such that $l \not\in S^\star_m(\lambda)$ and $k \in S^\star_m(\lambda)\;.$ The procedure is similar to what we did in best arm identification. Given two nonnegative sequences $(\delta_{n})_{n \in [M]}$ and $(\delta'_{n})_{n\in [M]}$, we define $\lambda =(\lambda_{k',n})_{k' \in [K]}$ such that
\[\left\{\begin{array}{ccl}
          \lambda_{k',n} &=& \mu_{k',n} \text{ if } k' \notin \{k, l\}\;, \\
          \lambda_{k,n} &= & \mu_{k,n} + \delta_{n} \;,\\
          \lambda_{l,n} & = & \mu_{l,n} - \delta'_{n}\;,
         \end{array}
 \right.\]
 and which satisfies 
	\begin{equation}(\lambda'_{k,m}-\mu'_{k,m}) - (\lambda'_{l,m}-\mu'_{l,m}) = \sum_{n \in [M]}w_{n,m}\left(\delta_{n} + \delta'_{n}\right) \geq \mu'_{l,m}-\mu'_{k,m}\;.\label{eq:constraint_topn}\end{equation}

	From Equation~\eqref{eq:cd_topn},
	\[\inf_{\delta,\delta' : \eqref{eq:constraint_topn} \text{ holds}} \left[\sum_{n} \tau_{k,n} \frac{\delta_{n}^2}{2} + \sum_{n} \tau_{l,n} \frac{(\delta'_{n})^2}{2}\right]\;.\] 

Solving this constrained optimization problem yields the following solution
	\[\delta_n = \frac{(\mu'_{l,m}-\mu'_{k,m}) w_{n,m}/\tau_{k,n}}{\sum_{n' \in [M]} w_{n',m}^2 \left(1/\tau_{k,n'} + 1/\tau_{l,n'}\right)} \text{ and } \ \delta'_n = \frac{(\mu'_{l,m}-\mu'_{k,m}) w_{n,m}/\tau_{l,n}}{\sum_{n' \in [M]} w_{n',m}^2 \left(1/\tau_{k,n'} + 1/\tau_{l,n'}\right)}\;.\]

We conclude similarly to the best arm identification case.
\end{proof}

Note that we retrieve the same bound as for the case $N=1$ (\ie best arm identification).

\paragraph{Relaxed Lower Bound Problem for Top-$N$ Identification} For Top-$N$ identification, let us define
	\[\forall k,m, \Delta^{'N}_{k,m} := \left\{\begin{array}{cl}
		\max^N_{k' \in [K]} \mu'_{k',m}-\mu'_{k,m} & \text{if } k\not\in S^\star_m\\
		\mu_{k,m}'- \max^{N+1}_{k' \in [K]} \mu'_{k',m} & \text{otherwise}\;,
         \end{array}
\right.\]
and $\widetilde{N}^\star_W(\mu)$ the value of problem $\oracle \left( \left( (\Delta^{'N}_{k,m})^2/2 \right)_{k,m} \right)\;.$ The set of constraints $\mathcal{N} := \left\{ t \mid \forall m \in [M], \forall k\not\in S^\star_m, \forall l\in S^\star_m, \sum_{n} w_{n,m}^2 \left( \frac{1}{t_{k,n}}+\frac{1}{t_{l,n}}\right) \leq \frac{(\mu'_{k,m}-\mu'_{l,m})^2}{2} \right\}$ is included in the set of constraints $\widetilde{\mathcal{N}} := \left\{ t \mid \forall m \in [M], \forall k \in [K], \sum_{n} \frac{w_{n,m}^2}{t_{k,n}} \leq \frac{\left(\Delta^{'N}_{k,m}\right)^2}{2} \right\}\;:$ indeed, if $t \in \mathcal{N}\;,$ then for any $m \in [M]$, and any $k \not\in S^\star_m\;,$

\[ \forall l \in S^\star_m\;, \ \sum_n \frac{w^2_{n,m}}{t_{k,n}} \leq \sum_n w_{n,m}^2\left(\frac{1}{t_{k,n}}+\frac{1}{t_{l,n}} \right) \leq \frac{(\mu'_{k,m}-\mu'_{l,m})^2}{2} \]
\[\implies \sum_n \frac{w_{n,m}^2}{t_{k,n}} \leq \min_{l \in S^\star_m} \frac{(\mu'_{k,m}-\mu'_{l,m})^2}{2} = \frac{(\mu'_{k,m}-\max^N_{k' \in [K]} \mu'_{k',m})^2}{2} = \frac{(\Delta^{'N}_{k,m})^2}{2}\;.\]

Similarly, for any agent $m$ and $l \in S^\star_m(\mu)$, one can check that $\sum_n \frac{w_{n,m}^2}{t_{l,n}} \leq \frac{(\Delta^{'N}_{l,m})^2}{2}\;,$ hence $t \in \widetilde{\mathcal{N}}\;.$ Then $N^\star_W(\mu) \geq \widetilde{N}^\star_W(\mu)\;.$


\paragraph{Algorithm for Top-$N$ identification} Algorithm~\ref{algo:fpe_bai} can then easily be adapted to Top-$N$ identification, with the following changes (the full algorithm is described in Algorithm~\ref{algo:fpe_topn})

\begin{algorithm*}[tb]
\caption{\textcolor{black}{Weighted Collaborative} 
Phased Elimination for Top-$N$ identification (\textcolor{black}{W-CPE}
-Top$N$)}
   \label{algo:fpe_topn}
\begin{algorithmic}
   \STATE {\bfseries Input:} $\delta \in (0,1)$, $M$ agents, $K$ arms, matrix $W$, $N \in [K]$
   \STATE Initialize $r \gets 0$, $\forall k,m, \widetilde{\Delta}_{k,m}(0) \gets 1, n_{k,m}(0) \gets 1$, $\forall m, B_m(0) \gets [K]$
   \STATE Draw each arm $k$ by each agent $m$ once
   \REPEAT
	\STATE \textcolor{gray}{\textcolor{black}{\# Central server}}
   \STATE $B(r) \gets \bigcup_{m \in [M]} B_m(r)$
	\STATE Compute $t(r) \gets \oracle\left(\left(\sqrt{2} \Deltat_{k,m}(r)\right)_{k,m}\right)$ 
	\STATE For all $k \in [K]$, compute \[(d_{k,m}(r))_{m\in [M]} \gets \arg\min_{d \in \Nn^{M}} \sum_{m} d_{m} \text{ s.t. } \forall m \in [M], \frac{n_{k,m}(r-1)+d_{m}}{\beta_{\delta}(n_{k,\cdot}(r-1)+d)} \geq t_{k,m}(r)\]
	\STATE \textcolor{black}{Send to each agent $m$ $(d_{k,m}(r))_{k,m}$ and $d_{\max} := \max_{n \in [M]} \sum_{k \in [K]} d_{k,n}(r)$}
     \STATE
	\STATE \textcolor{gray}{\textcolor{black}{\# Agent $m$}}
	\STATE Sample arm $k \in B(r)$ $d_{k,m}(r)$ times, so that $n_{k,m}(r) = n_{k,m}(r-1)+d_{k,m}(r)$
	\STATE \textcolor{black}{Remain idle for $d_{\max}-\sum_{k \in [K]} d_{k,m}(r)$ rounds}
	\STATE \textcolor{black}{Send to the server empirical mean $\hat{\mu}_{k,m}(r):=\sum_{s \leq n_{k,m}(r)} X_{k,m}(s)/n_{k,m}(r)$ for any $k \in [K]$}
    \STATE
 	\STATE \textcolor{gray}{\textcolor{black}{\# Central server}}
   \STATE Compute the empirical mixed means $(\hat{\mu}'_{k,m}(r))_{k,m}$ based on $(\hat{\mu}_{k,m}(r))_{k,m}$ and $W$
   \STATE {\textcolor{gray}{// \emph{Update set of candidate best arms for each user}}}
   \FOR{$m=1$ {\bfseries to} $M$} 
   \STATE \[B_{m}(r+1) \gets \left\{ k \in B_{m}(r) \mid \hat{\mu}'_{k,m}(r) + \Omega_{k,m}(r) \geq \max^N_{j \in B_m(r)} \left( \hat{\mu}'_{j,m}(r) - \Omega_{j,m}(r) \right) \right\}\]
   \ENDFOR
   \STATE {\textcolor{gray}{// \emph{Update the gap estimates}}}
   \STATE For all $k,m$, $\Deltat_{k,m}(r+1) \gets \Deltat_{k,m}(r) \times (1/2)^{\ind \left(k \in B_m(r+1) \land |B_m(r+1)|>N \right)}$
   \STATE $r \gets r+1$
   \UNTIL{$\forall m \in [M], |B_m(r)|\leq N$}
	\STATE {\bfseries Output:} $\left\{ k \in B_m(r) : m \in [M] \right\}$
\end{algorithmic}
\end{algorithm*}

	\paragraph{1.} Replace the stopping criterion (and the condition for the update of proxy gaps $(\Deltat_{k,m}(r))_{k,m}$ at round $r$) with \[ \forall m \in [M], |B_m(r)| \leq N \;,\]

	\paragraph{2.} Replace the elimination criterion with
			\[ B_m(r+1) \gets \left\{ k \in B_m(r) \mid \mu'_{k,m}+\Omega_{k,m}(r) \geq \max^N_{i \in B_m(r)} \left( \mu'_{i,m} - \Omega_{i,m}(r) \right) \right\}\;.\]

\begin{remark}
	Note that the proxy gap $\Deltat_{k,m}(r)$ no longer tracks the value of gap $\Delta'_{k,m}$ for $m \in [M], k \in [K]\;,$ but $\Delta^{'N}_{k,m}\;,$ and that on $N=1$, this algorithm exactly coincides with Algorithm~\ref{algo:fpe_bai}.
\end{remark}

\paragraph{Analysis of Algorithm~\ref{algo:fpe_topn}.} First, such an algorithm is indeed $\delta$-correct on event $\Ec$ (the same as defined for Algorithm~\ref{algo:fpe_bai}). Otherwise, for some agent $m\;,$ there would be an arm $l \in S^\star_m$ which is eliminated at round $r$ from $B_m(r+1)\;.$ But, on event $\Ec\;,$ Lemma~\ref{lem:validCI} implies that, for any $r \geq 0\;,$ $m \in [M]\;,$ and $(i,j) \in [K]^2\;,$

	\[ \hat{\mu}'_{i,m}(r)-\hat{\mu}'_{j,m}(r)+\Omega_{i,m}(r)+\Omega_{j,m}(r) \geq \mu'_{i,m}-\mu'_{j,m} \geq \hat{\mu}'_{i,m}(r)-\hat{\mu}'_{j,m}(r)-\Omega_{i,m}(r)-\Omega_{j,m}(r)\;.\] 

	Then, combining the right-hand inequality for $j=l$ with the elimination criterion
	\begin{eqnarray*}
		\max^N_{i \in [K]} \mu'_{i,m}-\mu'_{l,m} & \geq & \max^N_{i \in [K]} (\hat{\mu}'_{i,m}(r)-\hat{\mu}'_{l,m}(r) - \Omega_{i,m}(r)-\Omega_{l,m}(r))\\
		& \geq& \max^N_{i \in B_m(r) \subseteq [K]} (\hat{\mu}'_{i,m}(r)-\hat{\mu}'_{l,m}(r)-\Omega_{i,m}(r)-\Omega_{l,m}(r)) > 0\;,
	\end{eqnarray*}
	which is absurd because $l \in S^\star_m\;.$ Then, let us consider the following notation, for any $m\in [M]\;,$ $k \not\in S^\star_m\;,$

\[ R_{k,m} := \sup \{r \geq 0 : k \in B_m(r) \} \text{ and } r^N_{k,m} := \min \left\{ r \geq 0 : 4\Deltat_{k,m}(r) < \Delta^{'N}_{k,m}\right\}\;. \]

Note that the random number of rounds used by Algorithm~\ref{algo:fpe_topn} is then $R^N = \max_{m \in [M]} \max^N_{k \not\in S^\star_m} R_{k,m}\;.$ It is easy to prove that Lemma~\ref{lem:upperbound_omega} and Lemma~\ref{lem:deltat_value} still hold in Algorithm~\ref{algo:fpe_topn}. An equivalent result to Lemma~\ref{lem:round_elimination} can be shown

\begin{lemma}\label{lem:round_elimination_topn}
	On event $\Ec\;,$ for any $m \in [M]\;,$ $k \not\in S^\star_m\;,$ $R_{k,m} \leq r^N_{k,m}\;.$
\end{lemma}

\begin{proof}
	For any $m \in [M]\;,$ $k \not\in S^\star_m\;,$ $r=r_{k,m}\;,$ if $k \not\in B_m(r)\;,$ then the claim is true. Otherwise, if $k \in B_m(r)\;,$ then
\begin{eqnarray*}
	\hat{\mu}'_{k,m}(r) + \Omega_{k,m}(r) & \leq_{(1)}&  \mu'_{k,m}+2\Omega_{k,m}(r) \\
	&\leq_{(2)}& \mu'_{k,m}+4\Deltat_{k,m}(r)-2\Deltat_{k,m}(r)\\
	& <_{(3)} & \max^N_{i \in [K]} \mu'_{i,m}-2\Deltat_{k,m}(r) =_{(4)} \max^N_{i \in B_m(r)} \mu'_{i,m}-2\Deltat_{k,m}(r)\\ 
	& \leq_{(1)} & \max^N_{i \in B_m(r)} \left( \hat{\mu}'_{i,m}(r)-\Omega_{i,m}(r)+2\Omega_{i,m}(r) \right)-2\Deltat_{k,m}(r)\\
\end{eqnarray*}
Then
\begin{eqnarray*}
	\hat{\mu}'_{k,m}(r) + \Omega_{k,m}(r) &\leq_{(2)}& \max^{N}_{i \in B_m(r)} \left( \hat{\mu}'_{i,m}(r)-\Omega_{i,m}(r)+2\Deltat_{i,m}(r) \right) -2\Deltat_{k,m}(r)\\
	& \leq & \max^N_{i \in B_m(r)} \left( \hat{\mu}'_{i,m}(r) - \Omega_{i,m}(r) \right) + 2 \max^N_{i \in B_m(r)} \Deltat_{i,m}(r)-2\Deltat_{k,m}(r)\\
	& =_{(5)} & \max^N_{i \in B_m(r)} \left( \hat{\mu}'_{i,m}(r) - \Omega_{i,m}(r) \right) + 2\cdot 2^{-r} - 2 \cdot 2^{-r}\;,\\
	\implies \hat{\mu}'_{k,m}(r) + \Omega_{k,m}(r) & < & \max^N_{i \in B_m(r)} \left( \hat{\mu}'_{i,m}(r) - \Omega_{i,m}(r) \right) \;.\\
\end{eqnarray*}
	where $(1)$ is using by using event $\Ec$~; $(2)$ is using  Lemma~\ref{lem:upperbound_omega}~; $(3)$ uses $r=r^N_{k,m}$ and $k \not\in S^\star_m$~; $(4)$ is using event $\Ec$, and, for all $l \in S^\star_m$, $l \in B_m(r)$~; $(5)$ holds because Lemma~\ref{lem:deltat_value} is still valid and then, for all $j \in B_m(r)$, $\Deltat_{j,m}(r) = 2^{-r}\;.$
	Then $k$ is eliminated from $B_m(r)$ at round at most $r=r^N_{k,m}$, hence $R_{k,m} \leq r^N_{k,m}\;.$
\end{proof}

Using this result, the sample complexity analysis is the same as for best arm identification, which yields

\begin{theorem}
	With probability $1-\delta$, Algorithm~\ref{algo:fpe_topn} outputs the Top-$N$ arms for each agent using a total number of samples no greater than
	\[ \sup \{ n \in \Nn^{*} : n \leq 32  N^{\star}_W(\mu)\log_2(8/\Delta'_{\min})\beta^*(n) + KM\}\;, \]
	where $\beta^*(n) := \beta_{\delta}(n \mathds{1}_{[M]})\;.$
\end{theorem}

and we can use the same tools as in best arm identification to get an explicit upper bound depending on $N^{\star}_W(\mu)\;.$




\section{Experimental study}\label{sec:experiments}

The general weighted collaboration bandit framework has not been studied prior to this work. We investigate its performance in the special case of federated learning with personalization \cite{shi2021almost}, which corresponds to choosing weight matrices of the form $w_{n,m}=\alpha \ind{(n=m)} + \frac{1-\alpha}{M}$ for any pair of agents $(n,m)$. In this special case, we propose a baseline for weighted collaborative best arm identification which is a natural counterpart of the regret algorithm proposed by \cite{shi2021almost}, and compare it to our W-CPE-BAI algorithm. 

\subsection{A Simple BAI Algorithm Inspired by PF-UCB}\label{app:pfbai_baseline}

We state below as Algorithm~\ref{algo:pfucb_bai} a straightforward adaptation of the PF-UCB algorithm in~\cite{shi2021federated} to \textit{personalized} federated 
best arm identification (BAI)~; meaning that only weight matrices of the form $w_{n,m}=\alpha \ind{(n=m)} + \frac{1-\alpha}{M}$ for any pair of agents $(n,m)$ are considered. The original regret algorithm uses phased eliminations designed for each agent to identify their best arm together with \emph{exploitation}~: when all best arms have been found, or when some agent is waiting for others to finish their own exploration rounds, agents keep playing their empirical best arm. To turn this into a $\delta$-correct BAI algorithm, we remove the exploration rounds~; keep the same sampling rule within each phase (in which the number of samples from each arm is proportional to some rate function $f(r)$)~; and calibrate the size of the confidence intervals used to perform eliminations slightly differently, introducing for any $\delta \in (0,1)$ function
\[\forall r \geq 0, B_{r}(\delta) := \sqrt{\frac{2\log\left(KM\zeta(\beta)r^{\beta}/\delta\right)}{MF(r)}}\;.\]
where $F(r) = \sum_{p=1}^{r}f(p)$, for some $\beta > 1\;.$ In practice, we use $\beta=2\;.$ 

Algorithm~\ref{algo:pfucb_bai}, that we refer to as PF-UCB-BAI, follows the same general structure as our algorithm, with the notable difference that the number of samples of an arm $k \in B_m(r)$ in phase $r$ is fixed in advance. Under PF-UCB-BAI, when arm $k$ is still in the active set $B(r)$, agent $m$ 
\begin{itemize}
 \item performs global exploration to sample it $d_{k,m}^{g}(r) := \lceil(1-\alpha) f(r)\rceil$ times
 \item and additionally performs local exploration to sample it  $d_{k,m}^{\ell}(r) := \lceil\alpha M f(r)\rceil$ extra times if furthermore $k \in B_m(r)\;.$ 
\end{itemize}
Overall, $d_{k,m}(r) := d_{k,m}^{g}(r)+d_{k,m}^{\ell}(r) = \lceil (1-\alpha) f(r)\rceil\ind(k \in B(r)) + \lceil \alpha M f(r)\rceil\ind\left(k \in B_m(r)\right)$ new samples from arm $k$ are collected by agent $m$ during phase $r$ in order to update its estimate $\hat{\mu}_{k,m}(r)$ --the average of all available $n_{k,m}(r)$ samples for arm $k$ obtained by agent $m$-- which is sent to the central server. 


The mixed mean of each arm $(k,m)$ can then be computed by the server as
\[\hat{\mu}_{k,m}'(r) := \left(\alpha + \frac{1-\alpha}{M}\right)\hat{\mu}_{k,m}(r) + \frac{1-\alpha}{M} \sum_{m \neq n} \hat{\mu}_{k,n}(r)\;,\]
and sent back to each agent. We note that, in~\cite{shi2021federated}, they propose that the server computes the average $\hat{\mu}_k(r) := \frac{1}{M}\sum_{m=1}^{M}\hat{\mu}_{k,m}(r)$ across agents, and sends this value to each agent, who can then obtain $\hat{\mu}_{k,m}'(r) := \alpha \hat{\mu}_{k,m}(r) + (1-\alpha)\hat{\mu}_k(r)\;.$ 

\begin{algorithm*}[t]
\caption{PF-UCB-BAI}
   \label{algo:pfucb_bai}
\begin{algorithmic}
   \STATE {\bfseries Input:} $\delta \in (0,1)$, $M$ agents, $K$ arms, matrix $W$. 
   \STATE $f(r)$: sampling effort in phase $r$, $B_r(\delta)$: size of the confidence intervals in phase $r$.
   \STATE Initialize $r \gets 0\;,$ $\forall k,m, n_{k,m}(0) \gets 0\;,$ $\forall m, B_m(0) \gets [K]\;,$ $\hat{k}_m \gets 0\;.$ 
   \REPEAT 
	\STATE \textcolor{gray}{\textcolor{black}{\# Central server}}
   \IF{$|B_m(r)| = 1$}
   \STATE $\hat{k}_m \gets $ the unique arm in $B_m(r)$
   \STATE $B_m(r) = \emptyset$ 
   \ENDIF\color{black}
   \STATE $B(r) \gets \bigcup_{m \in [M]} B_m(r)$
    \FOR{$k \in B(r), m \in [M]$}
    \STATE $d_{k,m}(r) = \lceil (1-\alpha) f(r)\rceil + \lceil \alpha M f(r)\rceil\ind\left(k \in B_m(r)\right)$
    \ENDFOR
	\STATE \textcolor{black}{Send to each agent $m$ $(d_{k,m}(r))_{k,m}$ and $d_{\max} := \max_{n \in [M]} \sum_{k \in [K]} d_{k,n}(r)$}
    \STATE
	\STATE \textcolor{gray}{\textcolor{black}{\# Agent $m$}}
	\STATE Sample arm $k \in B(r)$ $d_{k,m}(r)$ times, so that $n_{k,m}(r) = n_{k,m}(r-1)+d_{k,m}(r)$
	\STATE \textcolor{black}{Remain idle for $d_{\max}-\sum_{k \in [K]} d_{k,m}(r)$ rounds}
	\STATE \textcolor{black}{Send to the server empirical mean $\hat{\mu}_{k,m}(r):=\sum_{s \leq n_{k,m}(r)} X_{k,m}(s)/n_{k,m}(r)$ for any $k \in [K]$}
    \STATE
	\STATE \textcolor{gray}{\textcolor{black}{\# Central server}}
   \STATE Compute the empirical mixed means $(\hat{\mu}'_{k,m}(r))_{k,m}$ based on $(\hat{\mu}_{k,m}(r))_{k,m}$ and $W$
   \STATE {\textcolor{gray}{// \emph{Update set of candidate best arms for each user}}}
   \FOR{$m=1$ {\bfseries to} $M$} 
   \STATE \[B_{m}(r+1) \gets \left\{ k \in B_{m}(r) \mid \hat{\mu}'_{k,m}(r) + B_{r}(\delta) \geq \max_{j \in B_m(r)} \left( \hat{\mu}'_{j,m}(r) - B_{r}(\delta) \right) \right\}\]
   \ENDFOR 
   \STATE $r \gets r+1$
   \UNTIL{$|B(r)| = \emptyset$}
	\STATE {\bfseries Output:} $\left\{ \hat{k}_m : m \in [M] \right\}$
\end{algorithmic}
\end{algorithm*}

Arm $k$ is eliminated from the active set $B_m(r)$ of agent $m$ if 
\[\hat{\mu}'_{k,m}(r) + B_{r}(\delta) < \max_{j \in B_m(r)} \left( \hat{\mu}'_{j,m}(r) - B_{p}(\delta) \right)\]
for the confidence parameter $B_r(\delta)$. In the original algorithm, $B_r(\delta)$ is replaced by some function of $r$ and $T$, however a simple adaptation of Lemma $1$ in~\cite{shi2021federated} (adding a union bound on $r \in \Nn$) yields the following result. Indeed, the original result crucially exploits the sampling rule, which we did not change. 

\begin{lemma}\label{lem:newlemma1} Event 
\[\mathcal{G} := \left\{\forall r \in \Nn^*, \forall m \in [M],\forall k \in B_m(r), \left|\hat{\mu}'_{k,m}(r) - \mu_{k,m}'\right|\leq B_r(\delta)\right\}\]
holds with probability $1-\delta$. 
\end{lemma}

On the good event $\mathcal{G}$ introduced in Lemma~\ref{lem:newlemma1}, observe that arm $k_m^*$ can never be eliminated from the set $B_m(r)$, therefore it has to be the guess $\hat{k}_m$ that agent $m$ outputs. This proves that Algorithm~\ref{algo:pfucb_bai} is $\delta$-correct for pure exploration for the special case of 
federated 
bandit with personalization. 
This algorithm can therefore serve as a baseline to be compared to our proposal in this particular case. 

\textcolor{black}{
We can also upper bound the sample complexity of this algorithm. Indeed, on event $\mathcal{G}$, like in the analysis of PF-UCB in~\cite{shi2021federated}, we can upper bound the number of rounds where arm $k$ is sampled by agent $m$ by $p_{k,m} := \inf \{r : B_r(\delta) \leq  \Delta'_{k,m}/4\}$. When $f(p) = 2^p$, one can prove that 
\[ \sum_{p=1}^{p_{k,m}} f(p) = \mathcal{O}\left(\frac{\log(1/\delta)}{M(\Delta'_{k,m})^{2}}\right)\;. \]
Summing the (deterministic) global and local exploration cost over rounds, arms and agents, yields an exploration cost of order
\[\mathcal{O}\left( \sum_{k \in [K]} \left[ \left( \frac{1-\alpha}{\min_{n \in [M]} (\Delta'_{k,m})^2} \right) + \left( \sum_{m \in [M]} \frac{\alpha}{(\Delta'_{k,m})^2} \right) \right]\log\left(\frac{1}{\delta}\right) \right) \;.\]
}

\subsection{Numerical experiments}\label{subsec:experiments}

As we did throughout the paper, we  consider Gaussian bandits with fixed variance $\sigma^2=1$. 
We denote $\hat{r}$ the average number of communication rounds across the $R$ iterations of an experiment~; $\hat{c}$ the average exploration cost of the considered algorithm across the $R$ iterations~; $\hat{\delta}$ the empirical error frequency across the $R$ iterations. $\hat{r}$ and $\hat{c}$ are reported $\pm$ their standard deviation rounded up to the closest integer, except for $\hat{\delta}$, which is rounded up to the $2^{\text{nd}}$ decimal place.

We consider a synthetic instance with $K=6$ arms, $M=3$ agents, for $R=100$ iterations. In order to generate randomly this instance, we sampled at random $K \times M$ values $x_{k,m}$ from the distribution $\mathcal{N}(0,1)\;,$ and set $\mu_{k,m} = x_{k,m}/\|x\|_{\mathcal{F}}\;,$~\footnote{$\|\cdot\|_{\mathcal{F}}$ is the Frobenius norm for matrices.} and tested if the associated $\Delta'_{\min}$ satisfied $\Delta'_{\min} \geq 0.05\;.$ We repeated this sampling until this condition was fulfilled. 

\paragraph{Comparison to PF-UCB-BAI} Our first experiment is to compare our Algorithm~\ref{algo:fpe_bai} with the PF-UCB-BAI baseline described above. For PF-UCB-BAI we use a phase length $f(p) := 2^p\log(1/\delta)$ for $p \geq 0\;$. For both algorithms, we set $\delta=0.1$ and experiment with $\alpha \in \{0.4,0.5,0.6,0.7\}$. Results are reported on the \textcolor{black}{left-most table} in Table~\ref{tab:pf_BAI}. In terms of communication cost, \textcolor{black}{W-CPE}
-BAI (Algorithm~\ref{algo:fpe_bai}) improves considerably over the baseline. Depending on the value $\alpha \in [0,1]$ (the closer it is to $1$, the less agents have to communicate in order to get good estimates of their mixed expected reward) \textcolor{black}{W-CPE}
-BAI improves or has an exploration cost which is similar to the baseline, up to a constant lower than $2$.

\paragraph{Comparison to an oracle algorithm} Our second experiment is to assess the asymptotic optimality (up to some logarithmic factors) of our algorithm. In order to estimate the scaling in $T^\star(\cdot)$, we have implemented an oracle algorithm 
which has access to the true gaps $\left( \left(\Delta'_{k,m}\right)_{k,m} \right)$ and can compute the associated $T^\star_W(\mu)$. Then we compare the average exploration cost $\hat{c}$ 
with $c^*:=\left\lceil T^\star_W(\mu)\log\left(\frac{1}{2.4\delta}\right) \right\rceil\;,$ for $\delta \in \{0.00001, 0.0001,0.001,0.01,0.05,0.1\}\;,$ and $\alpha=0.5\;.$ We reported the associated results in the \textcolor{black}{right-most table} in Table~\ref{tab:pf_BAI}. We can notice that, as $\delta$ decreases, the ratio $\hat{c}/c^*$ also decreases. As predicted by our upper bound in Theorem~\ref{th:sample_complexity_bai}, \textcolor{black}{W-CPE}
-BAI does not attain asymptotic optimality even for small values of $\delta$, but has a scaling to $T^\star_W(\mu)\log\left(\frac{1}{2.4\delta}\right)$ which decreases as $\delta$ goes to $0$.

\paragraph{Numerical considerations} These experiments were run on a personal computer (configuration: processor Intel Core i$7$-$8750$H, $12$ cores $@2.20$GHz, RAM $16$GB). In order to solve the optimization problem defining the oracle $t(r)$, 
we used CVXPy~\citep{diamond2016cvxpy,agrawal2018rewriting}, with the commercial solver MOSEK~\citep{mosek} tuned to default parameters. To compute the number of samples $(d_{k,m}(r))_{k,m,r}$, we used the \textit{optimize} module from SciPy~\citep{2020SciPy-NMeth}. 
In our experiments with this implementation of \textcolor{black}{W-CPE}
-BAI, we found that when the instance is too hard --meaning that the associated $\Delta'_{\min}$ is small-- the optimization part is subject to numerical approximation errors, which prevents the computation of the oracle allocation.  This might however be mitigated by online optimization approaches, such as in~\cite{degenne2020structure}.

\begin{table}[t]
	\caption{Personalized \textcolor{black}{collaborative} 
	BAI with varying $\alpha \in \{ 0.4, 0.5, 0.6, 0.7\}$~; the most efficient algorithm in terms of exploration cost is in bold type , all algorithms yield $\hat{\delta}=0$ up to the $5^{\text{th}}$ decimal place (\textit{top table}). Personalized \textcolor{black}{collaborative} 
	BAI with varying $\delta \in \{ 0.00001, 0.0001, 0.001, 0.01, 0.05, 0.1\}$~; ratios are rounded up to the $1^{\text{st}}$ decimal place.}
\label{tab:pf_BAI}
\vskip 0.15in
\begin{small}
\begin{sc}
\begin{tabular}{llccc}
\toprule
	$\alpha$ & Algorithm & $\hat{r}$ & $\hat{c}$ & $\hat{\delta}$ \\
\midrule
$0.4$ & \textbf{\textcolor{black}{W-CPE}
-BAI}  & \textbf{5$\pm$ 0} & \textbf{87,918$\pm$ 8,634} & \textbf{0.00}\\
& PF-UCB-BAI & 12$\pm$ 0 & 109,822$\pm$ 30,335 & 0.00\\
\midrule
$0.5$ & \textbf{\textcolor{black}{W-CPE}
-BAI} & \textbf{5$\pm$ 0} &  \textbf{75,094$\pm$ 7,596} & \textbf{0.00}\\
& PF-UCB-BAI & 11$\pm$ 0 & 73,561$\pm$ 17,239 & 0.00 \\
\midrule
$0.6$ & \textbf{\textcolor{black}{W-CPE}
-BAI}  & \textbf{5$\pm$ 0} & \textbf{78,334$\pm$ 7,983} & \textbf{0.00}\\
& PF-UCB-BAI & 11$\pm$ 0 & 55,812$\pm$ 14,793 & 0.00 \\
\midrule
$0.7$ &\textcolor{black}{W-CPE}
-BAI & 5$\pm$ 0 & 76,817$\pm$ 10,953 & 0.00\\
& \textbf{PF-UCB-BAI} & \textbf{10$\pm$ 0} & \textbf{45,765$\pm$ 9,591} & \textbf{0.00} \\
\bottomrule
\end{tabular}
\begin{tabular}{lcccr}
\toprule
	$\delta$ & $\hat{c}$ & $c^*$ & $\hat{c}/c^*$ \\
\midrule
0.1 & 75,094$\pm$ 7,596 & 1,104 & 68.0\\
0.05 & 75,992$\pm$ 8,468 & 1,639 & 46.4 \\
0.01 & 80,457$\pm$ 8,838 & 2,875 & 28.0\\
0.001 & 110,888$\pm$ 31,613 & 4,643 & 23.9\\
0.0001 & 91,692$\pm$ 9,121 & 6,411 & 14.3\\
0.00001 & 96,942$\pm$ 9,885 & 8,180 & 11.9\\
\bottomrule
\end{tabular}
\end{sc}
\end{small}
\vskip -0.1in
\end{table}

\end{document}